\documentclass{article}

\usepackage{microtype}
\usepackage{graphicx}
\usepackage{subcaption}
\usepackage{booktabs} % for professional tables

\usepackage{hyperref}

\usepackage[preprint]{icml2026}

\usepackage{amsmath}
\usepackage{amssymb}
\usepackage{mathtools}
\usepackage{amsthm}
\usepackage{enumitem}

\usepackage{multirow}
\usepackage{siunitx}
\usepackage{float}
\usepackage{afterpage}
\usepackage[table]{xcolor} % enables \rowcolor
\usepackage{caption}
\usepackage{MnSymbol,bbding,pifont}
\usepackage{wrapfig}
\usepackage{svg}
\usepackage[T1]{fontenc}
\usepackage{listings}
\usepackage{xcolor}
\usepackage{minted}

\usepackage[capitalize,noabbrev]{cleveref}

\theoremstyle{plain}
\newtheorem{theorem}{Theorem}[section]
\newtheorem{proposition}[theorem]{Proposition}
\newtheorem{lemma}[theorem]{Lemma}
\newtheorem{corollary}[theorem]{Corollary}
\theoremstyle{definition}

\newtheorem{assumption}[theorem]{Assumption}
\theoremstyle{remark}
\newtheorem{remark}[theorem]{Remark}

\usepackage[textsize=tiny]{todonotes}

\definecolor{GASPshade}{RGB}{220,235,255}

\makeatletter
\@ifundefined{ispreprint}{}{%
  \icmlshowauthorstrue

  \renewcommand{\Notice@String}{}%
  \def\@copyrightspace{}%

  \def\toptitlebar{\hrule height1pt \vskip 0.06in}%
  \def\bottomtitlebar{\vskip 0.06in \hrule height1pt \vskip 0.08in}%

  \renewcommand{\printAffiliationsAndNotice}[1]{%
    \global\icml@noticeprintedtrue%
    \stepcounter{@affiliationcounter}%
    \begingroup
      \footnotesize
      \par\centering
      #1\par
      \vspace{1pt}

      \noindent
      \forloop{@affilnum}{1}{\value{@affilnum} < \value{@affiliationcounter}}{%
        \textsuperscript{\arabic{@affilnum}}%
        \ifcsname @affilname\the@affilnum\endcsname%
          \csname @affilname\the@affilnum\endcsname%
        \else%
          {\bf AUTHORERR: Missing \string\icmlaffiliation.}%
        \fi%
        \ifnum\value{@affilnum}<\numexpr\value{@affiliationcounter}-1\relax
          \hspace{0.6em}{\raisebox{0.15ex}{\tiny$\bullet$}}\hspace{0.6em}%
        \fi
      }%
      \par

      \ifdefined\icmlcorrespondingauthor@text
        \vspace{1pt}\par
        Correspondence to: \icmlcorrespondingauthor@text.\par
      \fi
    \endgroup
  }%
}
\makeatother

\icmltitlerunning{Learning Robust Reasoning through Guided Adversarial Self-Play}

\begin{document}

\twocolumn[
  \icmltitle{Learning Robust Reasoning through Guided Adversarial Self-Play}

  % Local tightening UNDER the title (safer than moving the whole block)
  \vspace{-6pt}

  \icmlsetsymbol{equal}{*}
  % \vspace{-2pt}
  \begin{icmlauthorlist}
    \icmlauthor{Shuozhe Li}{ut}
    \icmlauthor{Vaishnav Tadiparthi}{hri}
    \icmlauthor{Kwonjoon Lee}{hri}
    \icmlauthor{Nakul Agarwal}{hri}
    \icmlauthor{Hossein Nourkhiz Mahjoub}{hri}
    \icmlauthor{Ehsan Moradi Pari}{hri}
    \icmlauthor{Lizhang Chen}{ut}
    \icmlauthor{Amy Zhang}{ut}
    \icmlauthor{Liu Leqi}{ut}
  \end{icmlauthorlist}

  \icmlaffiliation{ut}{University of Texas at Austin, Austin, TX, USA}
  \icmlaffiliation{hri}{Honda Research Institute USA, San Jose, CA, USA}

  \icmlcorrespondingauthor{Shuozhe Li}{shuozhe.li@utexas.edu}

  % Put this INSIDE the title block so it appears right under authors.
  \vspace{-2pt}
  \printAffiliationsAndNotice{}

  \icmlkeywords{Machine Learning, ICML}

  \vspace{-2pt}
  \begin{center}
    \begingroup
    \setkeys{Gin}{width=0.85\textwidth}%
    \includegraphics{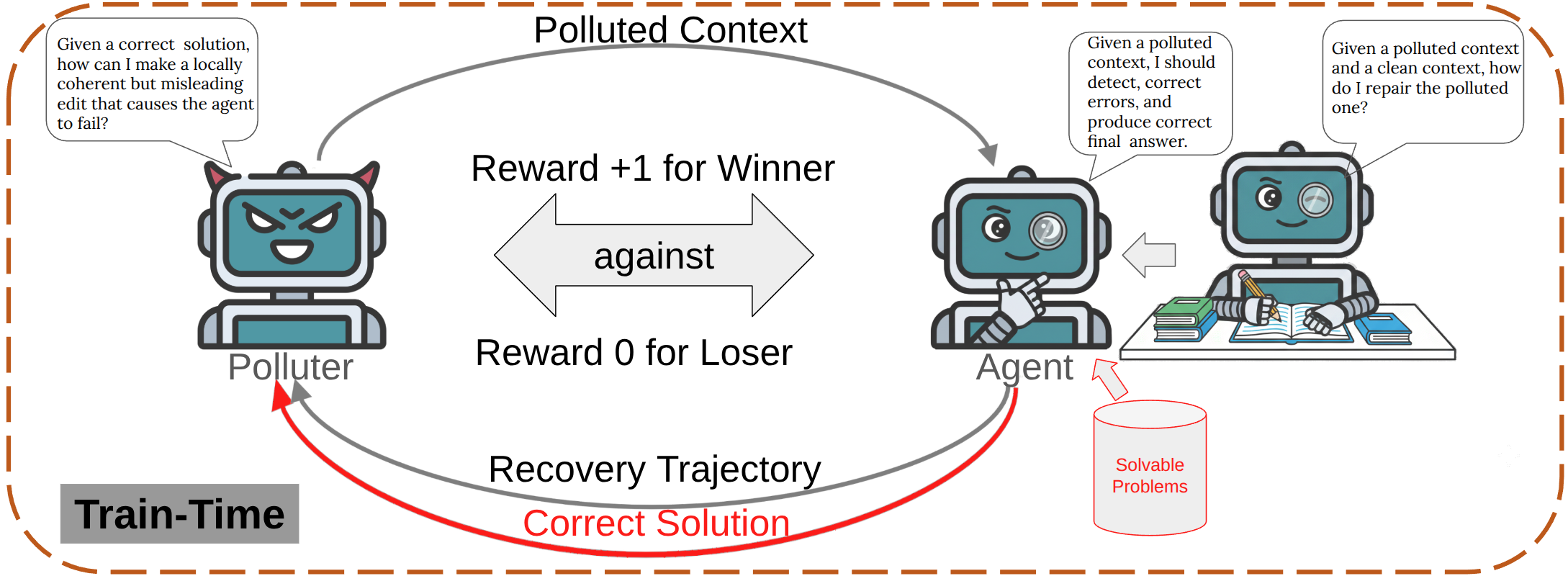}

    \captionsetup{width=0.95\textwidth,justification=raggedright,singlelinecheck=false}
    \vspace{-6pt}
    \captionof{figure}{\textsc{GASP} trains robust reasoners that resist misleading context, detect and repair errors, and answer reliably.}
    \label{fig:open_fig}
    \endgroup
  \end{center}
]

% NOTE: removed the \vspace{-10pt} here to avoid messing with global flow
\begin{abstract}
Reinforcement learning from verifiable rewards (RLVR) produces strong reasoning models, yet they can fail catastrophically when the conditioning context is fallible (e.g., corrupted chain-of-thought, misleading partial solutions, or mild input perturbations), since standard RLVR optimizes final-answer correctness only under clean conditioning. We introduce \textbf{GASP} (\underline{\textbf{G}}uided \underline{\textbf{A}}dversarial \underline{\textbf{S}}elf-\underline{\textbf{P}}lay), a robustification method that explicitly trains detect-and-repair capabilities using only outcome verification. Without human labels or external teachers, \textbf{GASP} forms an adversarial self-play game within a single model: a \textbf{polluter} learns to induce failure via locally coherent corruptions, while an \textbf{agent} learns to diagnose and recover under the same corrupted conditioning. To address the scarcity of successful recoveries early in training, we propose \emph{in-distribution repair guidance}, an imitation term on self-generated repairs that increases recovery probability while preserving previously acquired capabilities. Across four open-weight models (1.5B--8B), \textbf{GASP} transforms strong-but-brittle reasoners into robust ones that withstand misleading and perturbed context while often improving clean accuracy. Further analysis shows that adversarial corruptions induce an effective curriculum, and in-distribution guidance enables rapid recovery learning with minimal representational drift.
\end{abstract}
\begin{figure*}[t]
\centering
% \raggedright
\vspace{-8pt}
\includegraphics[width=0.85\textwidth]{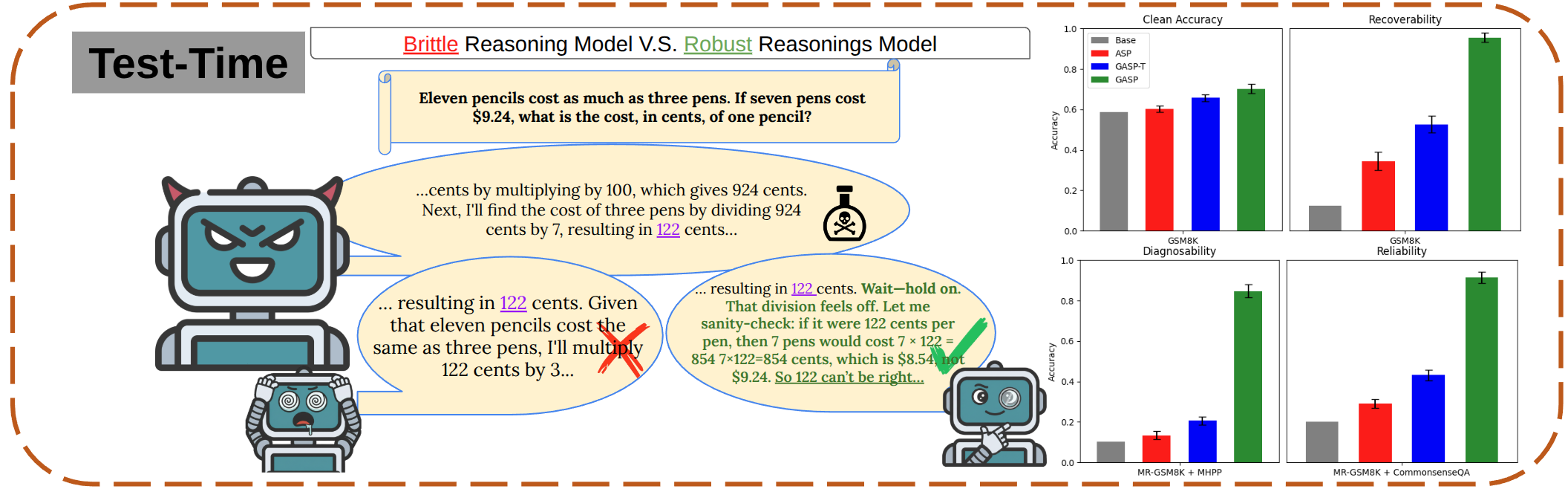}
%\vspace{-8.5pt}
\caption{A brittle model follows corruptions and fails, while a robust model detects the inconsistency, and repairs to reach the correct answer; corresponding results show \textsc{GASP} markedly improves recoverability, diagnosability, reliability, and clean accuracy.}
\label{fig:open_2}
\vspace{-15pt}
\end{figure*}
% \vspace{-30pt}
\section{Introduction}
\vspace{-5pt}
Reinforcement learning from verifiable rewards (RLVR) \citep{yu2025rlpr} has turned large language models into strong ``reasoning'' systems. On optimizing post-training only for final-answer correctness, recent large reasoning models (LRMs) such as DeepSeek-R1 \citep{guo2025deepseek} achieve near-saturated performance on math and coding benchmarks by producing long chain-of-thought (CoT) solutions. RLVR can also induce qualitative changes in behavior: models sometimes express uncertainty, double-check intermediate steps, and revise an answer mid-generation---``aha-moment'' self-reflection behaviors that correlate with higher success rates \citep{yang2025understanding}.

These behaviors are largely an \emph{incidental byproduct} of optimizing final-answer correctness: RLVR does not explicitly train \emph{when} to distrust the conditioning context, \emph{how} to diagnose inconsistencies, or \emph{how} to reliably recover when the context is polluted. However, in real deployments, models rarely operate in such perfectly reliable settings: the conditioning context---the prompt, a partial solution trace, or a collaborator’s reasoning in multi-agent scenarios \citep{cui2025communication}---may be noisy, misleading, or distribution-shifted, and the system must decide what to trust versus what to re-derive. \textbf{Recoverability} tests \citep{li2025off} show that, \textit{even on the correctly answered questions}, inserting a short distracting or locally corrupted snippet into an otherwise correct CoT can cause strong math models to fail catastrophically. Instead of questioning the injected step and \textbf{recover} from corruption, models often treat the visible trajectory as authoritative and follow the corruption; strikingly, this failure mode can exhibit inverse scaling \citep{li2025off}. \textbf{Diagnosability} benchmarks such as MR-GSM8K and MR-Bench \citep{zeng2023mr,zeng2024mr} reveal a complementary weakness under a different kind of fallible context: models that solve the problem well can collapse when asked to diagnose a provided solution, locate the first erroneous step, and explain the mistake. Beyond explicit reasoning traces, \textbf{reliability} benchmarks such as RUPBench \citep{wang2024rupbench} show that reasoning can also be brittle when the input context itself is perturbed by mild lexical, syntactic, and semantic edits, leading to substantial drops across commonsense and logical tasks.

Together, these stress-tests point to a common deficiency in robustness to conditioning context: many LRMs behave as strong \emph{clean} problem solver but are unreliable when the conditioning information is fallible. While emergent self-reflection indicates that models can sometimes notice issues and recover, current training pipelines do not target the capability we ultimately care about in noisy settings: \textbf{detect errors, repair them, and reach the correct answer}.
% \vspace{-5pt}

In this work, we ask: \textbf{Can we convert a strong-but-brittle LRM into a robust reasoner by training robustness to fallible conditioning context using only verifiable outcome rewards?} We focus on a regime where base task competence is already present (e.g., from RLVR or instruction tuning) and train a reasoning model robust to corrupted context. To do this without human annotations or external teacher models, we develop an \textbf{adversarial self-play} framework \citep{cheng2024self} for robust reasoning. We instantiate two role-conditioned behaviors over the same underlying model: a \textbf{polluter} that learns to introduce subtle corruptions of the conditioning context that maximize downstream failure, and an \textbf{agent} that learns to diagnose and neutralize these corruptions while preserving final-answer correctness. Both roles are optimized with a GRPO-style objective using only verifiable terminal rewards. Because the polluter trains against the current agent, self-play induces an adaptive curriculum: as the agent becomes harder to fool, the polluter must discover increasingly effective corruption patterns, and the agent correspondingly learns stronger verification and repair strategies.
Crucially, the agent is trained under the same corrupted conditioning used at test time, so robustness arises from learning rather than from specialized test-time prompting.
% \vspace{-4pt}

A central challenge is that recovery trajectories are initially rare: under fallible context, outcome-only policy optimization often sees batches with no successes, producing uninformative updates. To address this, we introduce \textbf{in-distribution repair guidance}: a lightweight imitation term on \emph{self-generated} repair snippets that are high-likelihood under the current policy. This increases early recovery rates and makes outcome-based RL updates informative without relying on off-distribution teacher fixes. Empirically, in-distribution guidance both accelerates recovery learning and better preserves previously acquired capabilities (smaller representation drift); we further isolate this mechanism with a minimal navigation analogue analysis (\S\ref{sec:maze_analogue}).

We call our training procedure \textbf{\textsc{GASP}} (\textbf{\underline{G}}uided \textbf{\underline{A}}dversarial \textbf{\underline{S}}elf-\textbf{\underline{P}}lay). Across four open-weight reasoning models (1.5B--8B), \textsc{GASP} transforms strong-but-brittle reasoners into robust ones: it improves \emph{diagnosability}, \emph{recoverability}, and \emph{reliability} under natural lexical/syntactic/semantic perturbations---while also strengthening \emph{self-reflection} (self-revision from its own incorrect solutions) and often improving \emph{clean} accuracy via more cautious reasoning.
Our contributions are:
% \vspace{-6pt}
\begin{itemize}[leftmargin=*,noitemsep]
    \vspace{-8pt}
    \item \textbf{Robust reasoning under fallible context.}
    We formalize robustness to unreliable conditioning context and unify three stress-tests: diagnosability, recoverability, and perturbation reliability.
    % \vspace{-8pt}
    \item \textbf{Adversarial self-play with verifiable rewards.}
    We introduce a two-role game where a polluter learns locally coherent corruptions that induce failure and an agent learns to detect-and-repair, using only terminal correctness rewards.
    % \vspace{-8pt}
    \item \textbf{In-distribution repair guidance.}
    We add a lightweight imitation term on self-generated repairs that increases early recovery rates.
    % \vspace{-8pt}
    \item \textbf{Empirical gains and analysis.}
    We show consistent robustness improvements across models, and find \textsc{GASP} also increases self-revision success and can improve clean accuracy; a navigation analogue and representation analyses suggest in-distribution guidance accelerates recovery learning under sparse outcome rewards while better preserving previously acquired capabilities.
\end{itemize}
% \vspace{-20pt}
\section{Preliminaries}
% \vspace{-8pt}
A problem instance is a question $q$ with a ground-truth final answer $a^\star$. A reasoning model $\pi_\theta$ generates a token sequence consisting of a chain-of-thought $c=(c_1,\dots,c_T)$ followed by a final answer $a$; we write a full generation as a trajectory $\tau=(c,a)\sim \pi_\theta(\cdot\mid q)$. For any ratio $\alpha\in[0,1]$, let $c_{0:\alpha}$ denote the prefix containing the first $\alpha$ fraction of tokens in $c$. \textbf{Verifiable Rewards}: all RL signals assume the existence of, and depend only on, a final-answer verifier using an outcome reward
$
    r(\tau; q,a^\star) \;=\; \mathbb{I}\{a=a^\star\}.
    % \label{verifiableReward}
$
\textbf{RL with verifiable rewards}:
% \vspace{-0.3cm}
We view text generation as an episodic decision-making process where the model emits tokens autoregressively until termination. The objective is to maximize expected terminal correctness:
$
J(\theta) \;=\; \mathbb{E}_{(q,a^\star)\sim \mathcal{D},\,\tau\sim \pi_\theta(\cdot\mid q)}\big[r(\tau;q,a^\star)\big].
$

% optionally with a KL regularizer to a reference policy.
\subsection{Group Relative Policy Optimization (GRPO)}
%\vspace{-8pt}
For a generic conditioning context $x$ (e.g., a question $q$, a polluted steer $s^{\mathrm{poll}}$, or a clean window context for the polluter), GRPO \citep{shao2024deepseekmath} samples a group of trajectories
$\{\tau_i\}_{i=1}^G \sim \pi_{\theta_{\mathrm{old}}}(\cdot \mid x)$
and assigns each trajectory a terminal score
$
R_i = \mathcal{R}(\tau_i; x),
$
where $\mathcal{R}$ is any task-specific terminal reward (e.g., verifiable correctness). Let
$\bar R=\mathrm{mean}(\mathbf{R})$ and $\sigma_R=\mathrm{std}(\mathbf{R})$ for $\mathbf{R}=\{R_1,\dots,R_G\}$.
We form group-relative advantages
$
\hat A_i(\mathcal{R}) = \frac{R_i-\bar R}{\sigma_R+\epsilon},
$
and broadcast them to all tokens:
$
\hat A_{i,t}(\mathcal{R})=\hat A_i(\mathcal{R}),\ \forall t\in\{1,\dots,|o_i|\}.
$
The GRPO objective parameterized by $\mathcal{R}$ is
% \vspace{-8pt}
\begin{align}
J_{\mathrm{GRPO}}&(\theta;\mathcal{R})
=\mathbb{E}\Bigg[
\frac{1}{G}\sum_{i=1}^G \frac{1}{|o_i|}\sum_{t=1}^{|o_i|}
\min\Big(
\rho_{i,t}(\theta)\hat A_{i,t}(\mathcal{R}), \nonumber\\
&\mathrm{clip}\big(\rho_{i,t}(\theta),1-\varepsilon,1+\varepsilon\big)\hat A_{i,t}(\mathcal{R})
\Big)
\Bigg].
\label{grpoObjective}
\end{align}
% \vspace{-20pt}
Optionally, we add a KL regularizer to a reference policy.
\section{Method}
\label{sec:method}
% \vspace{-8pt}
\subsection{\mbox{Overview: adversarial self-play for robust reasoning}}
\label{sec:overview}
% \vspace{-8pt}
Our goal is to convert a strong-but-brittle reasoner into a robust one that stays correct even when the conditioning context is unreliable. We operationalize robustness as \emph{diagnosability} (identify errors), \emph{recoverability} (repair and solve), and \emph{reliability} (resist perturbations). \textbf{GASP} trains these behaviors via adversarial self-play and outcome verification, rather than test-time prompting. Empirically, we find that models exhibiting stronger ``aha-moment'' behaviors—e.g., explicit uncertainty, checking, and revision—tend to be substantially more robust in these settings. Motivated by this connection, we target a concrete capability: \emph{given a partially corrupted chain-of-thought, produce the correct final answer}.
% \vspace{-6pt}

Notably, this capability is not entirely absent from current systems. Inference-time–scaled models trained with RLVR sometimes exhibit spontaneous self-correction and revision behaviors as an \emph{incidental byproduct} of optimizing for final-answer correctness. 
% However, these behaviors are neither explicitly trained nor reliably triggered: it remains unclear when they arise, what kinds of inconsistencies they can resolve, and how to systematically strengthen them. 
However, since these behaviors are not explicitly trained for, 
% As a result, 
% models that perform near-perfectly under clean solo reasoning 
models can still fail catastrophically when confronted with locally misleading intermediate context \citep{li2025off}.

We seek to train this capability using only verifiable outcome rewards, without human step-level labels and without any external editor/teacher model at training time. To this end, we formulate learning as a two-role adversarial self-play game over a shared trajectory. We build self-play episodes from questions the current agent solves \emph{reliably}, and sample a correct trace to serve as \emph{clean context}. This decouples robustness learning from competence acquisition: after we replace a local window with a coherent corruption, incorrect answers primarily reflect susceptibility to misleading conditioning rather than inability to solve $q$, details in \S\ref{sec:exp_setup}. A \textbf{polluter} then proposes a \emph{maliciously edited} version of this local window that is intended to induce downstream failure while remaining syntactically coherent. Conditioned on the resulting polluted context, the \textbf{agent} must {diagnose the inconsistency, repair the trajectory, and reach the correct final answer} under the same conditioning used at test time. 

% \vspace{-10pt}
\subsection{Constructing off-trajectory corruptions}
\label{sec:construction}
% \vspace{-8pt}
For each problem $(q,a^\star)$ that the agent solves reliably under solo reasoning, we first sample a correct trajectory
$
\tau = (c,a^\star) \sim \pi_\theta(\cdot \mid q),
$
where
$
c = (c_1,\dots,c_T).
$
We select a truncation ratio
$
\alpha \in \{0,\,0.25,\,0.5,\,0.75\},
$
and let $c_{0:\alpha}$ denote the prefix containing the first $\alpha$ fraction of tokens in $c$. This prefix represents the shared reasoning history that will be treated as fixed context. Following the truncation point, we extract a short \emph{clean window} $w^{\mathrm{clean}}$, consisting of the next $\approx 10\%$ of tokens from the same correct trajectory. This window serves as a local, ground-truth continuation that is internally consistent with both the problem and the prefix.

To construct an off-trajectory corruption, the clean window is replaced by a corrupted variant $w^{\mathrm{poll}}$ that is locally coherent in style but contains at least one misleading reasoning step and/or perturbed text. The resulting \emph{polluted steer} is
\[
s^{\mathrm{poll}} = \big(q,\; c_{0:\alpha},\; w^{\mathrm{poll}}\big),
\]
which removes the original clean continuation and exposes the agent to the corrupted context as shown in \cref{fig:open_2}.
% \vspace{-10pt}
% \subsection{Agent training objective: recover under the polluted steer}
% \label{sec:agent}
% % \vspace{-8pt}

% Conditioned on $s^{\mathrm{poll}}$, the agent produces a completion $\tau_{\mathrm{off}}\sim \pi_\theta(\cdot \mid s^{\mathrm{poll}})$ and receives only the verifiable terminal reward. % \eqref{verifiableReward}. 
% We optimize the agent with GRPO using \eqref{grpoObjective} under this off-trajectory conditioning,
% \vspace{-10pt}
% \[
% J_{\mathrm{GRPO}}^{\mathrm{rec}}(\theta)
% = 
% \mathbb{E} 
% % \Bigg],
% \Bigg[
% \frac{1}{G}\sum_{i=1}^G \frac{1}{|o_i|}\sum_{t=1}^{|o_i|}
% \min\!\Big(\rho_{i,t}(\theta)\hat A_{i,t},\; 
% \]
% \vspace{-12pt}
% \[
% \mathrm{clip}(\rho_{i,t}(\theta),1-\varepsilon,1+\varepsilon)\hat A_{i,t}\Big)
% \Bigg],
% \]

% where advantages $\hat A_{i,t}$ are computed from the group-relative normalization of terminal rewards. As we show next, optimizing this objective alone is typically insufficient when polluted contexts induce low-probability recovery behaviors.

\subsection{Agent training objective: recover under the polluted steer}
\label{sec:agent}
% \vspace{-8pt}

Conditioned on $s^{\mathrm{poll}}$, the agent produces an off-trajectory completion
$\tau_{\mathrm{off}}\sim \pi_\theta(\cdot \mid s^{\mathrm{poll}})$
and receives only a terminal verifiable reward. We therefore optimize the agent by instantiating the generic GRPO objective
$J_{\mathrm{GRPO}}(\theta;\mathcal{R})$ from \eqref{grpoObjective} with the recoverability reward
$\mathcal{R}^{\mathrm{rec}}(\tau_{\mathrm{off}}; s^{\mathrm{poll}}, a^\star)=\mathbb{I}\{a_{\mathrm{off}} = a^\star\}$ i.e.,
\[
J_{\mathrm{GRPO}}^{\mathrm{rec}}(\theta)
\;=\;
J_{\mathrm{GRPO}}(\theta;\mathcal{R}^{\mathrm{rec}}),
\]
where the group-relative advantages $\hat A_{i,t}(\mathcal{R}^{\mathrm{rec}})$ are computed from the terminal rewards
$\mathcal{R}^{\mathrm{rec}}(\tau_i; s^{\mathrm{poll}}, a^\star)$ within each sampled group. As we show next, optimizing outcome-only $J_{\mathrm{GRPO}}^{\mathrm{rec}}$ is typically insufficient when polluted contexts make recovery trajectories low probability.

\subsection{Repair guidance under learning-signal scarcity}
\label{sec:guidance}

\paragraph{Motivation: outcome-only learning is signal-scarce under pollution.}
Under a polluted steer $s^{\mathrm{poll}}$, the agent initially almost never recovers: outputs typically follow the misleading window and terminate with an incorrect answer. This creates a \emph{learning-signal scarcity} problem for outcome-only GRPO, since group-relative advantages are informative only when % at least
some sampled trajectories succeed. 
Increasing the group size $G$ raises the chance of observing a %rare 
recovery but is computationally expensive and does not change how rarely recovery behaviors are sampled; see \S\ref{app:learning_signal_scarcity}. This motivates a complementary mechanism that directly increases the probability of sampling successful recoveries under $s^{\mathrm{poll}}$, so informative groups arise even with moderate $G$.

\paragraph{Guided objective.}
To increase early recovery frequency without increasing $G$, we augment outcome-only GRPO with an auxiliary \emph{repair guidance} term that upweights short ``repair'' continuations back to a successful trajectory:
\[
J_{\mathrm{G\text{-}GRPO}}^{\mathrm{rec}}(\theta)
\;=\;
J_{\mathrm{GRPO}}^{\mathrm{rec}}(\theta)
\;+\;
\lambda\, J_{\mathrm{guide}}(\theta).
\]
Here $J_{\mathrm{GRPO}}^{\mathrm{rec}}$ is instantiated as in \S\ref{sec:agent}, and $J_{\mathrm{guide}}$ is an imitation-style objective defined on repair snippets under the \emph{deployment} conditioning $s^{\mathrm{poll}}$.

\paragraph{Teacher-guided repairs can be off-distribution.}
A straightforward approach uses a strong external teacher policy $\pi_{\mathrm{T}}$ to generate a repair snippet
$w^{\mathrm{fix}}_{\mathrm{T}} \sim \pi_{\mathrm{T}}(\cdot \mid p)$
from a diagnosis prompt $p$ that exposes the inconsistency (e.g., including both $w^{\mathrm{clean}}$ and $w^{\mathrm{poll}}$), and then behavior-clones this snippet under the deployment conditioning $s^{\mathrm{poll}}$.
Empirically, teacher fixes can be semantically correct yet inefficient as learning signals because they often lie in low-probability regions of the student policy; see \S\ref{app:ood_repair_guidance}.

\paragraph{Advantage of in-distribution guidance.}
For a generic policy gradient on a positive sample $w^{\mathrm{fix}}$ under the deployment conditioning $s^{\mathrm{poll}}$, we have 
$
g(w^{\mathrm{fix}})=\omega \nabla_\theta \log \pi_\theta(w^{\mathrm{fix}}\mid s^{\mathrm{poll}})
$
with weight $\omega>0$.
A first-order improvement satisfies
\vspace{-5pt}
\begin{align*}
\Delta J
&\;\approx\;
\alpha\, \nabla_\theta J(\theta)^\top g(w^{\mathrm{fix}})
\\
&\;\approx\;
\alpha\,\omega\,\pi_\theta(w^{\mathrm{fix}}\mid s^{\mathrm{poll}})\,
\big\|\nabla_\theta \log \pi_\theta(w^{\mathrm{fix}}\mid s^{\mathrm{poll}})\big\|_2^2,
\end{align*}
where the dominant term scales linearly with $\pi_\theta(w^{\mathrm{fix}}\mid s^{\mathrm{poll}})$ (full proof in \S\ref{app:proof_guidance_scaling}). Therefore, for candidate repairs $w_\theta^{\mathrm{fix}},w_{\mathrm{T}}^{\mathrm{fix}}$ with comparable weights and gradient norms,
\vspace{-5pt}
\[
\pi_\theta(w_\theta^{\mathrm{fix}}\mid s^{\mathrm{poll}})\;\gg\;\pi_\theta(w_{\mathrm{T}}^{\mathrm{fix}}\mid s^{\mathrm{poll}})
\]
\vspace{-18pt}
\[
\quad\Longrightarrow\quad
|\Delta J(w_\theta^{\mathrm{fix}})|\;\gg\;|\Delta J(w_{\mathrm{T}}^{\mathrm{fix}})|.
\]

\begin{figure}[h]
    \vspace{-10pt}
    \centering
    \includegraphics[width=0.5\textwidth]{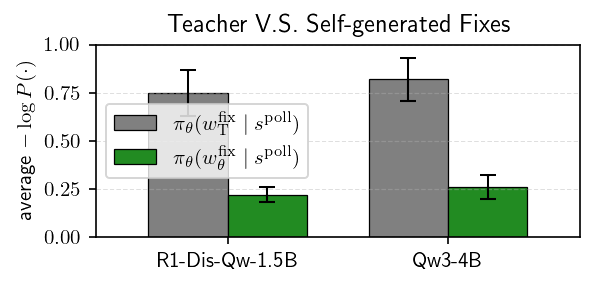}
    \vspace{-22pt}
    \caption{For fixes with similar lengths, teacher (GPT-5) \citep{singh2025openai} generated fixes have higher negative log-likelihood than self-generated fixes, indicating that the model assigns lower probability to teacher fixes than to its own fixes.}
    \label{fig:likelihood}
    \vspace{-15pt}
\end{figure}

\paragraph{Self-generated in-distribution repair targets.}
We construct a diagnosis prompt $p_\theta^{\mathrm{fix}}$ that includes $(q,c_{0:\alpha})$ and contrasts $(w^{\mathrm{clean}},w^{\mathrm{poll}})$, and sample an on-policy repair
$
w_\theta^{\mathrm{fix}} \sim \pi_\theta(\cdot\mid p_\theta^{\mathrm{fix}}).
$
We then use $w_\theta^{\mathrm{fix}}$ as an imitation target under the deployment conditioning
$s^{\mathrm{poll}}=(q,c_{0:\alpha},w^{\mathrm{poll}})$
by upweighting $\log \pi_\theta(w_\theta^{\mathrm{fix}}\mid s^{\mathrm{poll}})$.
By construction, this tends to yield repairs with substantially higher likelihood under deployment than teacher fixes:
$
\pi_\theta(w^{\mathrm{fix}}_{\theta}\mid s^{\mathrm{poll}})\gg \pi_\theta(w^{\mathrm{fix}}_{\mathrm{T}}\mid s^{\mathrm{poll}}),
$
so the same-sized update yields a larger expected improvement, leading to better learning efficiency (verified in \S\ref{sec:experiments}).

\paragraph{Representation preservation.}
Beyond immediate first-order gain, in-distribution repair targets tend to better preserve internal representations. Since self-generated repairs $w^{\mathrm{fix}}_\theta$ already have non-negligible likelihood under $s^{\mathrm{poll}}$ (Figure~\ref{fig:likelihood}), increasing $\log \pi_\theta(w^{\mathrm{fix}}_\theta \mid s^{\mathrm{poll}})$ typically requires only less parameter change \citep{zhu2025path}. By contrast, cloning an off-distribution teacher fix with $\pi_\theta(w^{\mathrm{fix}}_{\mathrm{T}} \mid s^{\mathrm{poll}})\approx 0$ demands a much larger likelihood-ratio adjustment, which can induce larger gradient bursts and greater latent drift in SFT-like updates \citep{huan2025does}. We verify this stability/retention effect in \S\ref{sec:experiments}.

% \vspace{-10pt}
\subsection{Adversarial self-play: optimizing the polluter}\label{sec:polluter}
% \vspace{-8pt}
So far, $w^{\mathrm{poll}}$ was an externally constructed corruption. We now generate $w^{\mathrm{poll}}$ online in a self-play loop. For each reliably solved problem $(q,a^\star)$, we first sample a clean trajectory $\tau=(c,a^\star)\sim\pi_\theta(\cdot\mid q)$, choose $\alpha$, and extract the clean window $w^{\mathrm{clean}}$ as in Section~\ref{sec:construction}. We then \emph{prompt the same model} to act as a \emph{polluter}: conditioned on the local clean context $(q,c_{0:\alpha},w^{\mathrm{clean}})$, it outputs a corrupted window
$
w^{\mathrm{poll}} \sim \pi_\theta(\cdot \mid q, c_{0:\alpha}, w^{\mathrm{clean}}; \mathrm{polluter\_prompt}),
$
% \vspace{-13pt}
which defines the polluted steer $s^{\mathrm{poll}}=(q,c_{0:\alpha},w^{\mathrm{poll}})$.
\paragraph{Polluter reward.}
After sampling $w^{\mathrm{poll}}$, we run the \emph{agent} under the polluted steer and evaluate the same verifiable correctness reward:
\[
\tau_{\mathrm{off}} \sim \pi_\theta(\cdot \mid s^{\mathrm{poll}}),
\qquad
r_{\mathrm{solve}}=\mathbb{I}\{a_{\mathrm{off}}=a^\star\}.
\]
The polluter receives reward when its corruption causes failure:
$
R^{\mathrm{poll}} \;=\; 1 - r_{\mathrm{solve}}.
$
% \vspace{-10pt}
\paragraph{Polluter objective.}
We optimize the polluter by instantiating the generic GRPO objective $J_{\mathrm{GRPO}}(\theta;\mathcal{R})$ from \eqref{grpoObjective} on the polluted-window trajectories.
For each clean context $x^{\mathrm{poll}}=(q,c_{0:\alpha},w^{\mathrm{clean}})$, the polluter samples a window
$w^{\mathrm{poll}}\sim \pi_\theta(\cdot \mid x^{\mathrm{poll}};\mathrm{polluter\_prompt})$,
which induces the polluted steer $s^{\mathrm{poll}}=(q,c_{0:\alpha},w^{\mathrm{poll}})$.
After rolling out the agent under $s^{\mathrm{poll}}$, we assign the polluter the terminal reward
$
\mathcal{R}^{\mathrm{poll}}(w^{\mathrm{poll}}; x^{\mathrm{poll}}, a^\star)
=
1-\mathbb{I}\{a_{\mathrm{off}}=a^\star\},
\tau_{\mathrm{off}}\sim \pi_\theta(\cdot\mid s^{\mathrm{poll}}).
$
We then update the polluter with
\[
J_{\mathrm{GRPO}}^{\mathrm{poll}}(\theta)
\;=\;
J_{\mathrm{GRPO}}(\theta;\mathcal{R}^{\mathrm{poll}}),
\]
where the group-relative advantages $\hat A_{i,t}(\mathcal{R}^{\mathrm{poll}})$ are computed from the rewards
$\mathcal{R}^{\mathrm{poll}}(w^{\mathrm{poll}}_i; x^{\mathrm{poll}}, a^\star)$
over a group of $G_{\mathrm{poll}}$ sampled windows for the same $x^{\mathrm{poll}}$.

% \vspace{-10pt}
\paragraph{Overall self-play updates (\textsc{GASP}).}
Training alternates two GRPO-style updates using the same parameters $\theta$: (i) an \emph{agent} update that maximizes recoverability under polluted context (GRPO plus the in-distribution repair guidance term from Section~\ref{sec:guidance}), and (ii) a \emph{polluter} update that maximizes $J_{\mathrm{GRPO}}^{\mathrm{poll}}$. Concretely, we alternate ascent steps
\vspace{-5pt}
\[
\theta \leftarrow \theta + \eta_{\mathrm{rec}} \nabla_\theta J_{\mathrm{G-GRPO}}^{\mathrm{agent}}(\theta),
\hspace{0.5em}
\theta \leftarrow \theta + \eta_{\mathrm{poll}} \nabla_\theta J_{\mathrm{GRPO}}^{\mathrm{poll}}(\theta),
\]
where $J_{\mathrm{G-GRPO}}^{\mathrm{agent}}(\theta)$ denotes the guided objective (e.g., $J_{\mathrm{GRPO}}^{\mathrm{rec}}(\theta)+\lambda J_{\mathrm{guide}}(\theta)$). When optimizing the polluter, we treat $r_{\mathrm{solve}}$ as a black-box outcome of the agent rollout and do not backpropagate through $\tau_{\mathrm{off}}$. This self-play induces an adaptive curriculum: as the agent becomes harder to fool, the polluter must discover increasingly effective local corruptions, and the agent correspondingly learns stronger error-detection and repair behaviors under the same deployment conditioning $s^{\mathrm{poll}}$.

\newcommand{\meanstd}[2]{\ensuremath{#1 \pm #2}}
\newcommand{\score}[1]{\ensuremath{#1}}
\newcommand{\na}{\textemdash}

\begin{table*}[t]
\centering
\footnotesize
\setlength{\tabcolsep}{2.6pt}
\renewcommand{\arraystretch}{0.82}
\resizebox{\textwidth}{!}{%
\begin{tabular}{ll c c c c c c c c}
\toprule
& &
\multicolumn{1}{c}{\textbf{Clean Accuracy}} &
\multicolumn{1}{c}{\textbf{Recoverability}} &
\multicolumn{2}{c}{\textbf{Diagnosability}} &
\multicolumn{2}{c}{\textbf{Reliability}} &
\multicolumn{1}{c}{\textbf{Self-Revision}} &
\multicolumn{1}{c}{\textbf{Avg}} \\
\cmidrule(lr){3-3}
\cmidrule(lr){4-4}
\cmidrule(lr){5-6}
\cmidrule(lr){7-8}
\cmidrule(lr){9-9}
\cmidrule(lr){10-10}

\textbf{Model} & \textbf{Method} &
{GSM8K} &
{GSM8K} &
{MR-GSM8K} & {MHPP} &
{GSM8K} & {CQA} &
{$\Delta$ GSM8K} &
{Avg} \\
\midrule

\multirow{4}{*}{R1-Distill-Qwen-1.5B}
& Initial ckpt        & \score{58.7} & \score{12.4}  & \score{12.1}  & \score{8.4}  & \score{24.3} & \score{16.2} & \score{0.2} & 18.9 \\
& \textsc{ASP}        & \meanstd{60.2}{1.6} & \meanstd{34.4}{4.5} & \meanstd{15.8}{2.1} & \meanstd{11.2}{1.8} & \meanstd{38.7}{2.3} & \meanstd{19.5}{2.1} & \meanstd{2.1}{0.3} & 26.0 \\
& \textsc{GASP-T}     & \meanstd{65.8}{1.7} & \meanstd{52.7}{4.2} & \meanstd{24.1}{2.2} & \meanstd{17.3}{1.9} & \meanstd{57.5}{2.7} & \meanstd{28.8}{2.3} & \meanstd{2.8}{0.4} & 35.6 \\
\rowcolor{GASPshade}\cellcolor{white}\strut
& \textsc{GASP}       & \meanstd{70.3}{2.4} & \meanstd{95.6}{2.2} & \meanstd{87.4}{3.1} & \meanstd{82.1}{3.5} & \meanstd{91.2}{2.8} & \meanstd{91.7}{2.9} & \meanstd{6.1}{0.5} & \textbf{74.9} \\
\midrule

\multirow{4}{*}{DeepScaleR-1.5B}
& Initial ckpt        & \score{67.9} & \score{14.3} & \score{13.8}  & \score{9.1}  & \score{25.6} & \score{17.4} & \score{0.2} & 21.2 \\
& \textsc{ASP}        & \meanstd{69.2}{1.2} & \meanstd{36.7}{5.1} & \meanstd{18.3}{2.4} & \meanstd{13.5}{2.0} & \meanstd{41.4}{2.5} & \meanstd{22.1}{2.3} & \meanstd{1.9}{0.3} & 29.0 \\
& \textsc{GASP-T}     & \meanstd{70.8}{1.4} & \meanstd{56.0}{4.6} & \meanstd{27.8}{2.6} & \meanstd{20.6}{2.1} & \meanstd{61.8}{2.9} & \meanstd{33.1}{2.5} & \meanstd{2.6}{0.4} & 39.1 \\
\rowcolor{GASPshade}\cellcolor{white}\strut
& \textsc{GASP}       & \meanstd{72.4}{2.1} & \meanstd{98.6}{1.1} & \meanstd{90.2}{2.7} & \meanstd{85.8}{3.2} & \meanstd{94.3}{2.1} & \meanstd{92.6}{2.4} & \meanstd{8.2}{0.6} & \textbf{77.4} \\
\midrule

\multirow{4}{*}{Qwen3-4B}
& Initial ckpt        & \score{77.2} & \score{13.8}  & \score{15.4} & \score{11.2}  & \score{27.9} & \score{19.8} & \score{2.3} & 23.9 \\
& \textsc{ASP}        & \meanstd{79.2}{1.1} & \meanstd{40.5}{5.2} & \meanstd{22.7}{2.6} & \meanstd{17.9}{2.3} & \meanstd{45.6}{2.8} & \meanstd{26.3}{2.5} & \meanstd{4.6}{0.5} & 33.8 \\
& \textsc{GASP-T}     & \meanstd{81.4}{1.2} & \meanstd{61.4}{4.0} & \meanstd{34.9}{2.7} & \meanstd{27.4}{2.4} & \meanstd{67.5}{2.6} & \meanstd{39.1}{2.7} & \meanstd{6.3}{0.6} & 45.4 \\
\rowcolor{GASPshade}\cellcolor{white}\strut
& \textsc{GASP}       & \meanstd{82.9}{1.9} & \meanstd{98.2}{1.2} & \meanstd{92.8}{2.3} & \meanstd{88.5}{2.9} & \meanstd{96.1}{1.8} & \meanstd{94.7}{2.1} & \meanstd{9.3}{0.7} & \textbf{80.4} \\
\midrule

\multirow{4}{*}{Qwen3-8B (lora)}
& Initial ckpt            & \score{84.9} & \score{14.7}  & \score{18.2} & \score{13.5}  & \score{33.3} & \score{22.1} & \score{4.2} & 27.3 \\
& \textsc{ASP}            & \meanstd{87.3}{1.5} & \meanstd{44.4}{2.1} & \meanstd{26.9}{2.8} & \meanstd{21.4}{2.5} & \meanstd{59.8}{2.9} & \meanstd{30.2}{2.7} & \meanstd{6.4}{0.6} & 39.5 \\
& \textsc{GASP-T}  & \meanstd{89.9}{1.6} & \meanstd{66.9}{2.4} & \meanstd{40.7}{2.9} & \meanstd{32.5}{2.6} & \meanstd{88.6}{2.7} & \meanstd{44.1}{2.6} & \meanstd{8.8}{0.7} & 53.2 \\
\rowcolor{GASPshade}\cellcolor{white}\strut
& \textsc{GASP}    & \meanstd{93.7}{3.2} & \meanstd{98.7}{0.8} & \meanstd{95.1}{1.9} & \meanstd{91.6}{2.4} & \meanstd{97.8}{1.1} & \meanstd{96.3}{1.6} & \meanstd{10.7}{0.8} & \textbf{83.4} \\
\bottomrule
\end{tabular}
}

\vspace{2pt}
\caption{\textbf{Main robustness results (pass@1).} Mean $\pm$ std over 3 training runs for \textsc{ASP}/\textsc{GASP-T}/\textsc{GASP}; \textbf{Initial ckpt} is a single evaluation. \textbf{Clean Accuracy}: GSM8K under clean prompting. \textbf{Recoverability}: GSM8K accuracy under corrupted context (held-out GPT-5 polluter). \textbf{Diagnosability}: \textbf{ACCreason} on MR-GSM8K and MHPP \citep{zeng2023mr,zeng2024mr}. \textbf{Reliability}: accuracy on perturbed GSM8K and CommonsenseQA from RUPBench \citep{wang2024rupbench}. \textbf{Self-Revision}: improved GSM8K accuracy on the failed questions after self-revision \citep{tie2025can}. \textbf{Avg}: unweighted mean of the seven metrics. \textsc{GASP-T} uses teacher-generated (GPT-5) repair guidance in place of self-guided in-distribution repairs. Lora \citep{hu2022lora} were used for 8B model see details in \cref{app:hyperparams}.}
\label{tab:main}
\vspace{-17pt}
\end{table*}
% \vspace{-12pt}
\section{Experiments}
\label{sec:experiments}
% \vspace{-8pt}
We evaluate \textsc{GASP} as a training procedure for turning strong-but-brittle LRMs into robust reasoners using only outcome reward without teacher/human labels. Our experiments are organized around three questions:
% \vspace{-14pt}
\begin{itemize}[leftmargin=*,noitemsep]
    \item \textbf{Q1: Robustness \textnormal{and} downstream utility.}
Does \textsc{GASP} improve \emph{diagnosability}, \emph{recoverability}, and \emph{reliability}, and do these robustness gains translate into better clean accuracy and \emph{self-revision} (post-hoc correction when prompted with the model's own incorrect solution)?
% \vspace{-14pt}
\item \textbf{Q2: Efficiency.}
Under a fixed rollout budget, does in-distribution repair guidance reduce the exploration burden for agent—producing recoveries sooner and reaching higher performance faster—than outcome-only GRPO and teacher-guided repair?
% \vspace{-14pt}
\item \textbf{Q3: Representation Drift.}
Does in-distribution guidance reduce representation drift and better preserve out-of-domain performance relative to off-distribution teacher guidance?
\end{itemize}

% \vspace{-25pt}
\subsection{Experimental setup}
\label{sec:exp_setup}
% \vspace{-8pt}
\paragraph{Models.}
We run experiments on four open-weight instruction-tuned reasoning models spanning 1.5B--8B parameters: DeepSeek-R1-Distill-Qwen-1.5B, DeepScaleR-1.5B, Qwen3-4B, and Qwen3-8B. Note, these are already strong reasoning and instruction following models, which makes them well-suited for role-conditioned self-play: under a \textsc{pollute} prompt, the same backbone can reliably generate plausible corruptions that induce failures in the \textsc{agent} role, creating a non-degenerate training game with informative adversaries.
% \vspace{-12pt}
\paragraph{Training questions and solvable subset.}
We sample 7.5k problems from the MATH training split and restrict training to instances the \emph{initial checkpoint} solves reliably: for each candidate, we draw $K{=}4$ solutions at temperature $0.7$ and keep it only if all final answers are correct. For each retained $(q,a^\star)$ during training, we sample a clean on-policy trajectory $\tau=(c,a^\star)\sim\pi_\theta(\cdot\!\mid\! q)$, choose $\alpha\in\{0,0.25,0.5,0.75\}$, and extract the subsequent clean window $w^{\mathrm{clean}}$ (the next $\approx 10\%$ of tokens, capped to a fixed length). The polluter generates $w^{\mathrm{poll}}$ online from this context, and the agent learns recovery under the polluted steer $s^{\mathrm{poll}}=(q,c_{0:\alpha},w^{\mathrm{poll}})$.
% \vspace{-12pt}
\paragraph{Prompts and role conditioning.}
Both roles are implemented by prompting the same backbone model. The \textsc{agent} uses the standard deployment system prompt \emph{unchanged} (no added “be robust” or “check for errors” instructions), while \textsc{pollute} uses a prompt with a constrained edit budget that encourages non-gibberish but misleading corruptions. Full prompts are in \S\ref{app:prompts}; reported robustness gains come from training, not test-time prompt engineering.
% \vspace{-12pt}
\paragraph{Baselines.}
In the main results table, we focus on two reference points that directly isolate the effect of our guidance mechanism: (i) the initial checkpoint (no additional training), and (ii) \textsc{GASP} trained with outcome-only GRPO under polluted steers (i.e., adversarial self-play \emph{without} any repair guidance). Additional comparisons---including GPT-5 polluter and ablations that disable polluter updates (fixed corruptions)--- are reported in \S\ref{app:baselines}. All methods are run under matched rollout budgets by controlling group sizes and maximum generation lengths.
% \vspace{-12pt}

%\paragraph{Implementation details.}
Unless otherwise stated, we use the same truncation set $\alpha\in\{0,0.25,0.5,0.75\}$, window-length cap, and final-answer verifier across all methods.
Full hyperparameters (GRPO clip, KL coefficient if used, group sizes, learning rates, and training steps) are given in \S\ref{app:hyperparams}.

% \vspace{-10pt}
\begin{figure*}[!htbp]
\centering
\begin{subfigure}[b]{0.20\linewidth}
  \centering
  \includegraphics[width=\linewidth]{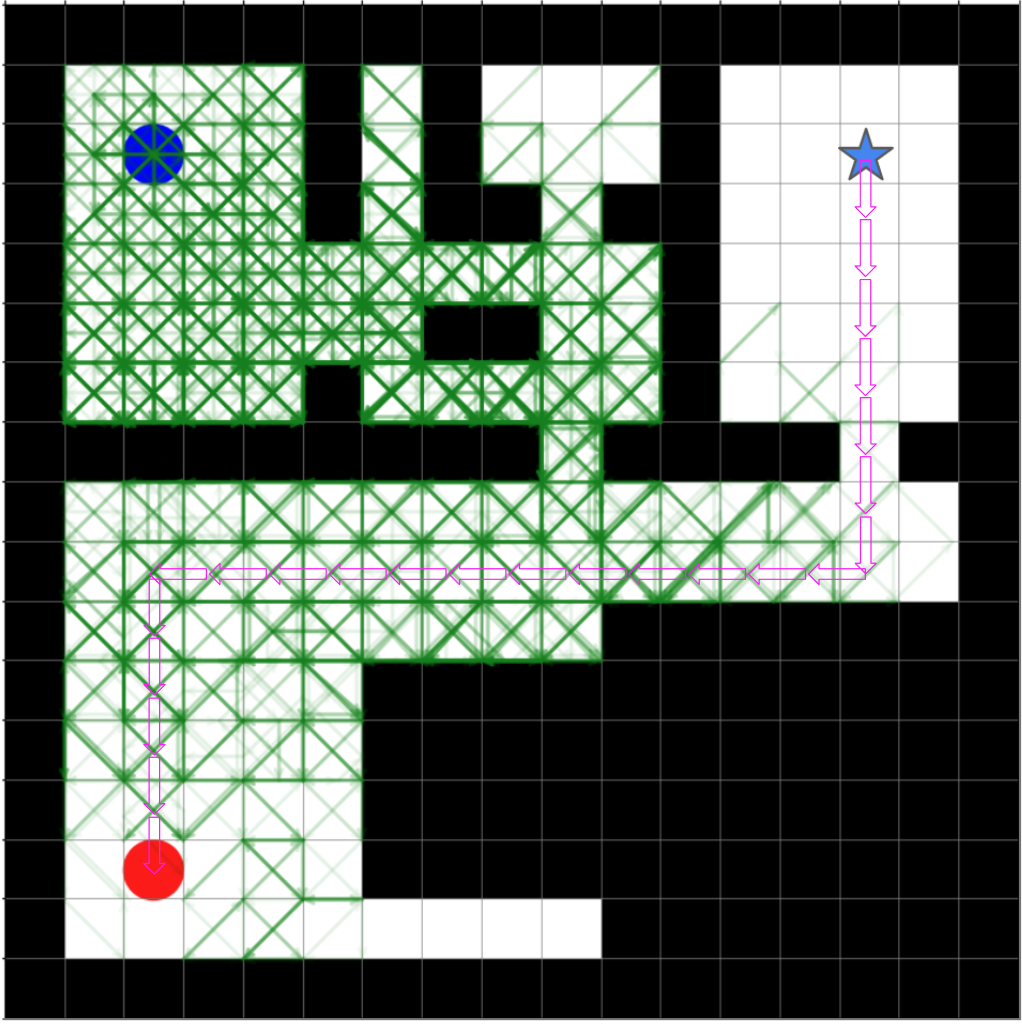}
  \caption{\scriptsize GRPO}
  \label{fig:toycase_grpo}
\end{subfigure}\hfill
\begin{subfigure}[b]{0.20\linewidth}
  \centering
  \includegraphics[width=\linewidth]{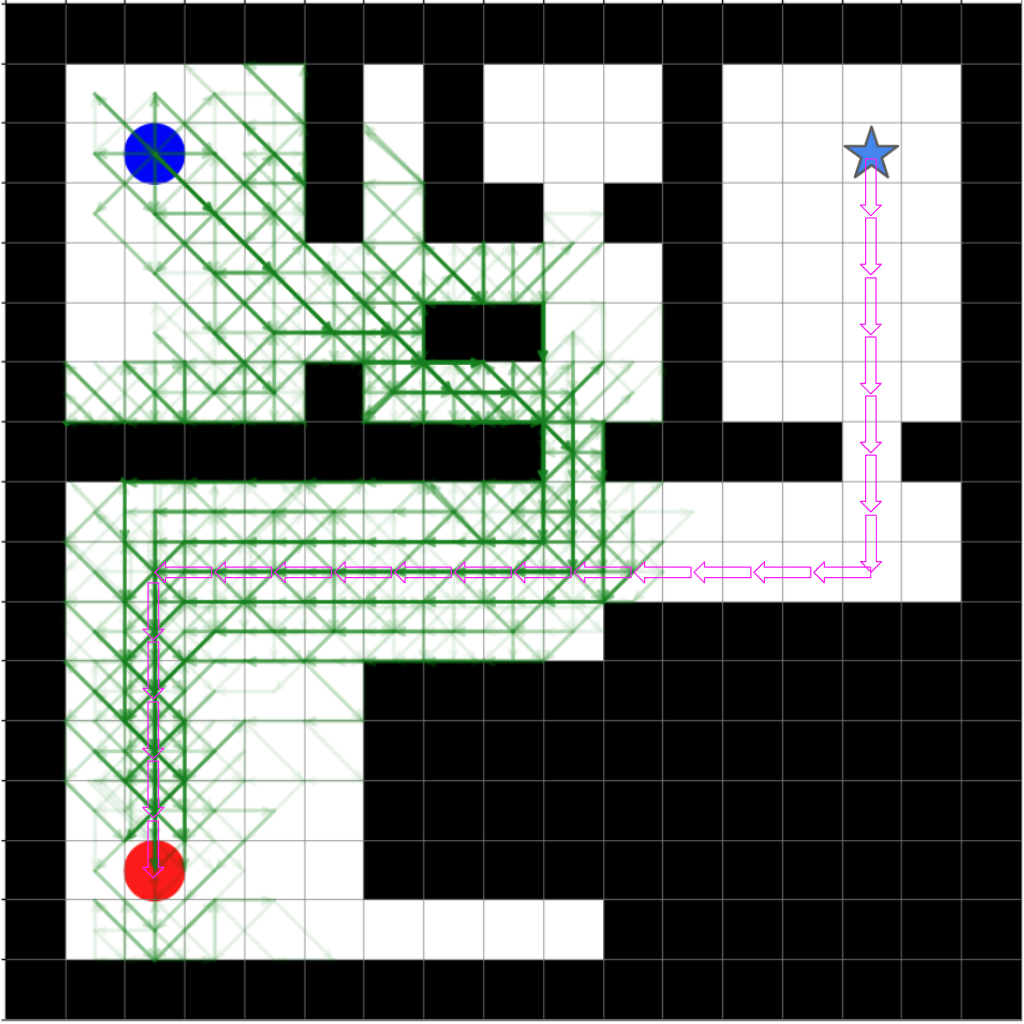}
  \caption{\scriptsize GRPO (In-Dist Guided)}
  \label{fig:toycase_g_grpo}
\end{subfigure}\hfill
\begin{subfigure}[b]{0.28\linewidth}
  \centering
  \includegraphics[width=\linewidth]{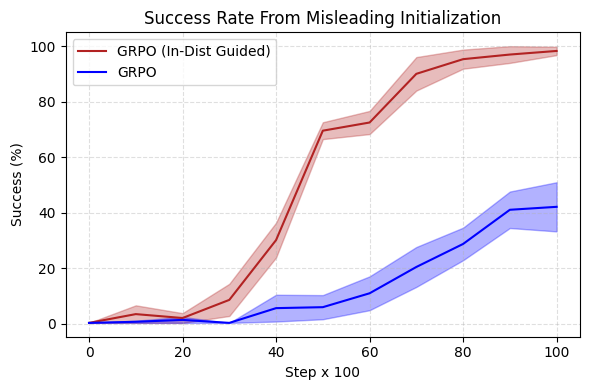}
  \caption{}
  \label{fig:mislead_maze}
\end{subfigure}\hfill
\begin{subfigure}[b]{0.28\linewidth}
  \centering
  \includegraphics[width=\linewidth]{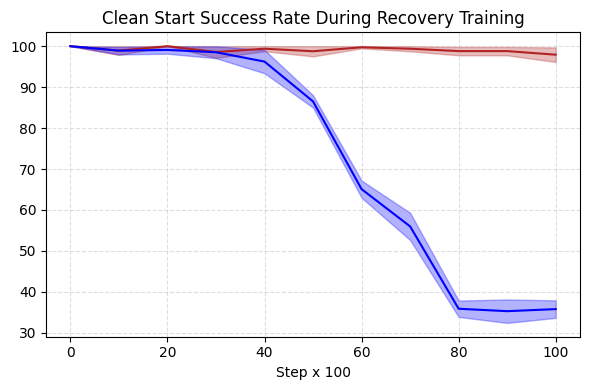}
  \caption{}
  \label{fig:clean_maze}
\end{subfigure}
\vspace{-7pt}
\caption{\textbf{Guidance accelerates off-trajectory recovery and preserves retention in a maze analogue.}
(a,b) Trajectories during recovery training from a misleading start (blue) to the goal (red); the star denotes the original clean start used to train the initial rail policy. GRPO-only explores broadly and drifts, while in-distribution guidance quickly returns to the rail and exploits prior knowledge. (c) Success rate from the misleading start; (d) clean-start success during recovery training (retention). Curves show mean $\pm$ std over 5 seeds, each evaluated with 10 rollouts per checkpoint.}
\label{fig:toycase}
 \vspace{-10pt}
\end{figure*}
\subsection{Evaluation benchmarks and metrics}
\label{sec:eval_suite}
% \vspace{-8pt}
We evaluate four complementary capabilities and summarize the main results in Table~\ref{tab:main}.
To isolate robustness to fallible context from base-task competence, we report \emph{Recoverability}, \emph{Diagnosability}, and \emph{Reliability} primarily on a \emph{clean-solved subset} of each benchmark: for every instance, we sample $K{=}4$ solutions under the standard (clean) prompt (temperature $0.7$) and retain it only if all final answers are correct. This conditioning ensures that failures reflect collapse under misleading context rather than inability to solve the underlying problem. For completeness, we also report performance on the full benchmarks in Appendix~\S\ref{app:full_benchmark_eval}.
% \vspace{-10pt}
\paragraph{Recoverability.}
We evaluate robustness to locally misleading intermediate context on GSM8K.
On the clean-solved subset, we sample a correct solution, truncate its chain-of-thought at
$\alpha\in\{0,0.25,0.5,0.75\}$, and replace the subsequent window with a locally coherent corruption generated by a held-out polluter (GPT-5).
We then measure final-answer accuracy under the resulting polluted steer.
% \vspace{-10pt}
\paragraph{Diagnosability.}
We evaluate whether the model can judge a provided solution, identify the first erroneous step, and explain the error.
On MR-GSM8K and MR-Ben's coding subset (MHPP), we follow the benchmark protocols and report \textbf{ACCreason}
(first-error-step + correct explanation) \citep{zeng2023mr,zeng2024mr}.
Exact prompting and answer extraction are in \S\ref{app:meta_reasoning}.
% \vspace{-10pt}
\paragraph{Reliability.}
We evaluate robustness to natural lexical/syntactic/semantic perturbations using the GSM8K and CommonsenseQA subsets of RUPBench \citep{wang2024rupbench}.
% \vspace{-16pt}
\paragraph{Self-revision.}
We measure whether a model can correct its own mistakes when shown its erroneous solution. On GSM8K, we first run the model once under the standard (clean) prompt and collect instances it answers incorrectly. For each such instance, we provide the original question together with the model's generated (incorrect) solution and prompt it to revise and output a corrected final answer, following \citet{tie2025can}. We report the resulting accuracy on the model's incorrect subset.

% \vspace{-10pt}
\subsection{A navigation analogue: guidance increases recovery probability and prevents drift}
\label{sec:maze_analogue}
% \vspace{-8pt}

To isolate the optimization roles of the outcome-RL term and our in-distribution repair guidance, we construct a minimal navigation analogue of guided learning. The key phenomenon we want to reproduce is: a policy that already knows an optimal solution from a clean start can nevertheless struggle to \emph{recover} once placed on a misleading suboptimal trajectory, because successful recovery rollouts are initially too rare to provide informative gradient update under outcome-only reward.
% \vspace{-15pt}

\paragraph{Environment.}
We use a deterministic grid maze with 8-connected actions (up/down/left/right and diagonals) and sparse terminal reward: an episode receives reward $1$ iff the agent reaches the target within a horizon $H$, and $0$ otherwise. Figure~\ref{fig:toycase_grpo} shows the layout. The \textbf{clean start} is the {\color{blue}\LARGE$\filledstar$} in the top-right; the \textbf{goal} is the {\color{red}\LARGE$\bullet$} in the bottom-left. The \textbf{misleading start} ({\color{blue}\LARGE$\bullet$}) lies on a suboptimal corridor that is locally plausible but requires backtracking through a narrow junction to rejoin the optimal route.
% \vspace{-10pt}

\paragraph{Two-stage protocol (``know the rail'' $\rightarrow$ ``recover from off-rail'').}
We first train a base policy $\pi_0$ with GRPO from the clean start ({\color{blue}\LARGE$\filledstar$}) to the goal ({\color{red}\LARGE$\bullet$}). This yields a policy that reliably executes a consistent \emph{rail} (pink arrows) under clean starts. We then switch the initial state distribution to the misleading start {\color{blue}\LARGE$\bullet$} and continue training from $\pi_0$ to solve the \emph{recovery} task ({\color{blue}\LARGE$\bullet$} $\rightarrow$ {\color{red}\LARGE$\bullet$}).
% \vspace{-10pt}
\paragraph{Methods compared.}
We compare two training variants under the misleading-start:
(i) \textbf{GRPO-only}, which optimizes the sparse terminal reward under the misleading start state, and (ii) \textbf{GRPO + in-distribution guidance (ours)}, which adds a lightweight imitation term on \emph{self-generated repair snippets}. Concretely, let $\mathcal{R}$ denote the set of states visited by $\pi_0$ when rolling out from the clean start, whenever an on-policy rollout reaches $\mathcal{R}$ (back to rail), we extract that trajectory segment and add it to a replay buffer $\mathcal{B}_{\mathrm{fix}}$. We then optimize an auxiliary behavioral-cloning objective on minibatches from $\mathcal{B}_{\mathrm{fix}}$:
\vspace{-14pt}
\[
J_{\mathrm{guide}}(\theta)
= \underset{\tau\sim \mathcal{B}_{\mathrm{fix}}}{\mathbb{E}}
\Big[\sum_{t=0}^{|\tau|}\log \pi_\theta(a_t\mid s_t)\Big],
\]
\vspace{-20pt}

The overall objective matches our method: $J(\theta)=J_{\mathrm{GRPO}}(\theta)+\lambda J_{\mathrm{guide}}(\theta)$. Because repair snippets are harvested from \emph{on-policy rollouts}, they are in-distribution under the current policy.
% \vspace{-10pt}
\paragraph{Metrics.}
We report (Figure~\ref{fig:mislead_maze}) \textbf{success rate} from the misleading start during training, and (Figure~\ref{fig:clean_maze}) \textbf{retention} by periodically evaluating success from the original clean start; this detects whether recovery training induces drift that damages the previously learned optimal behavior.
% \vspace{-10.25pt}
\paragraph{Results and interpretation.}
With \emph{GRPO-only} (Figure~\ref{fig:toycase_grpo}), the policy spends substantial time exploring within the misleading region before discovering the narrow rejoin junction; even after rejoining the rail, it frequently fails to leverage the previously learned route and keeps wandering—evidence of forgetting of the clean-start policy during recovery training. With \textbf{in-distribution guidance} (Figure~\ref{fig:toycase_g_grpo}), the agent rapidly learns a short ``repair'' maneuver that returns to $\mathcal{R}$ and then efficiently stitches the previously learned rail to the goal. Consistently, the learning curve in Figure~\ref{fig:mislead_maze} shows that guidance achieves high success much earlier than GRPO-only.
\begin{figure*}[!htbp]
\centering
\begin{subfigure}[b]{0.24\linewidth}
  \centering
  \includegraphics[width=\linewidth]{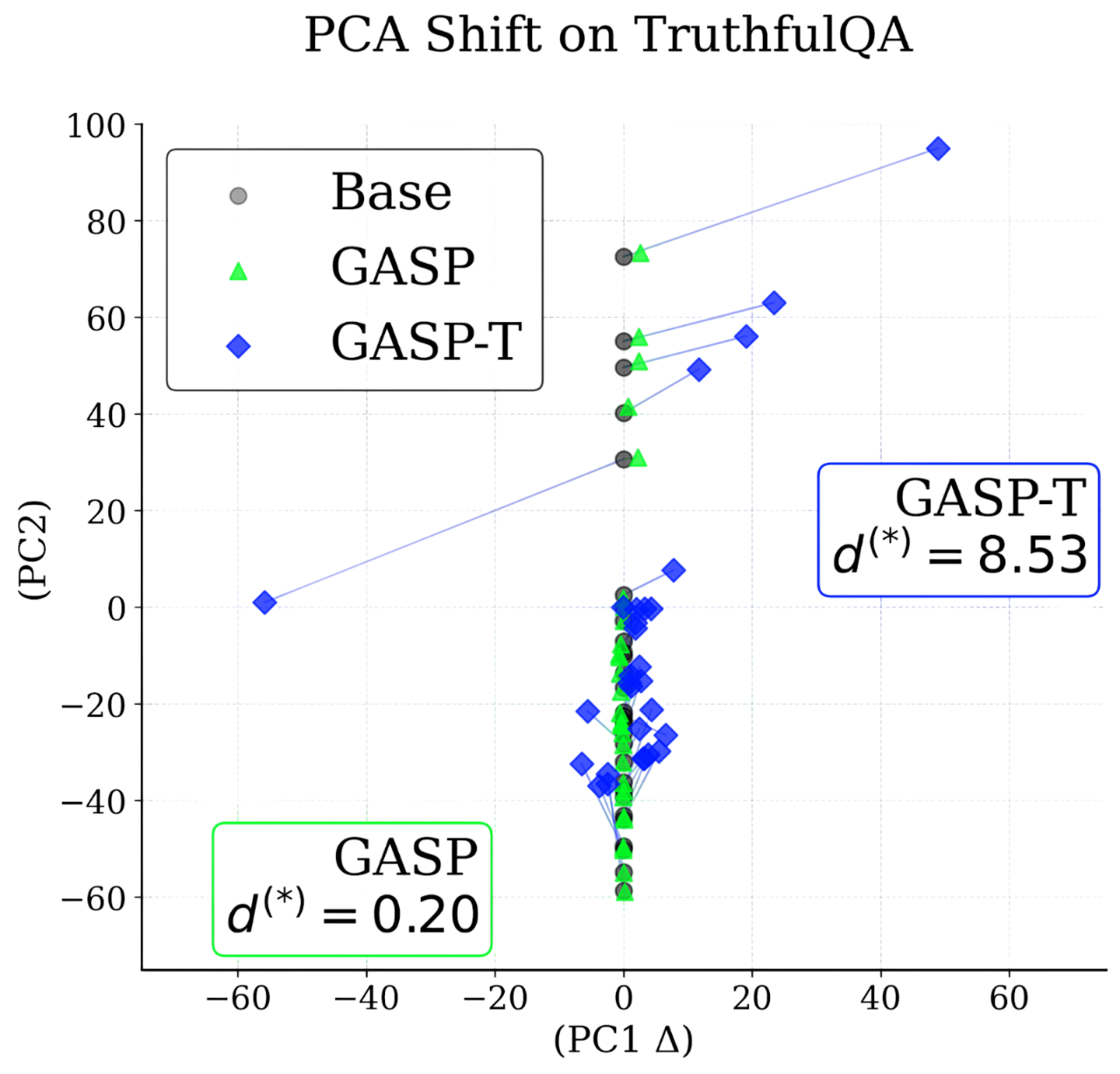}
  \vspace{-10pt}
  \caption{}
  \label{fig:toycase_grpo}
\end{subfigure}\hfill
\begin{subfigure}[b]{0.24\linewidth}
  \centering
  \includegraphics[width=\linewidth]{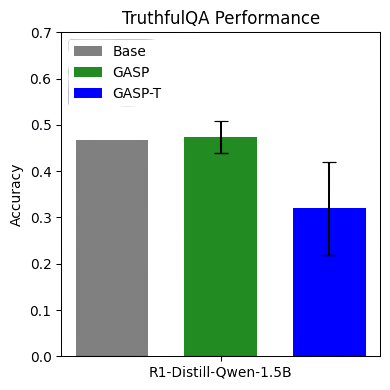}
  \vspace{-10pt}
  \caption{}
  \label{fig:toycase_g_grpo}
\end{subfigure}\hfill
\begin{subfigure}[b]{0.24\linewidth}
  \centering
  \includegraphics[width=\linewidth]{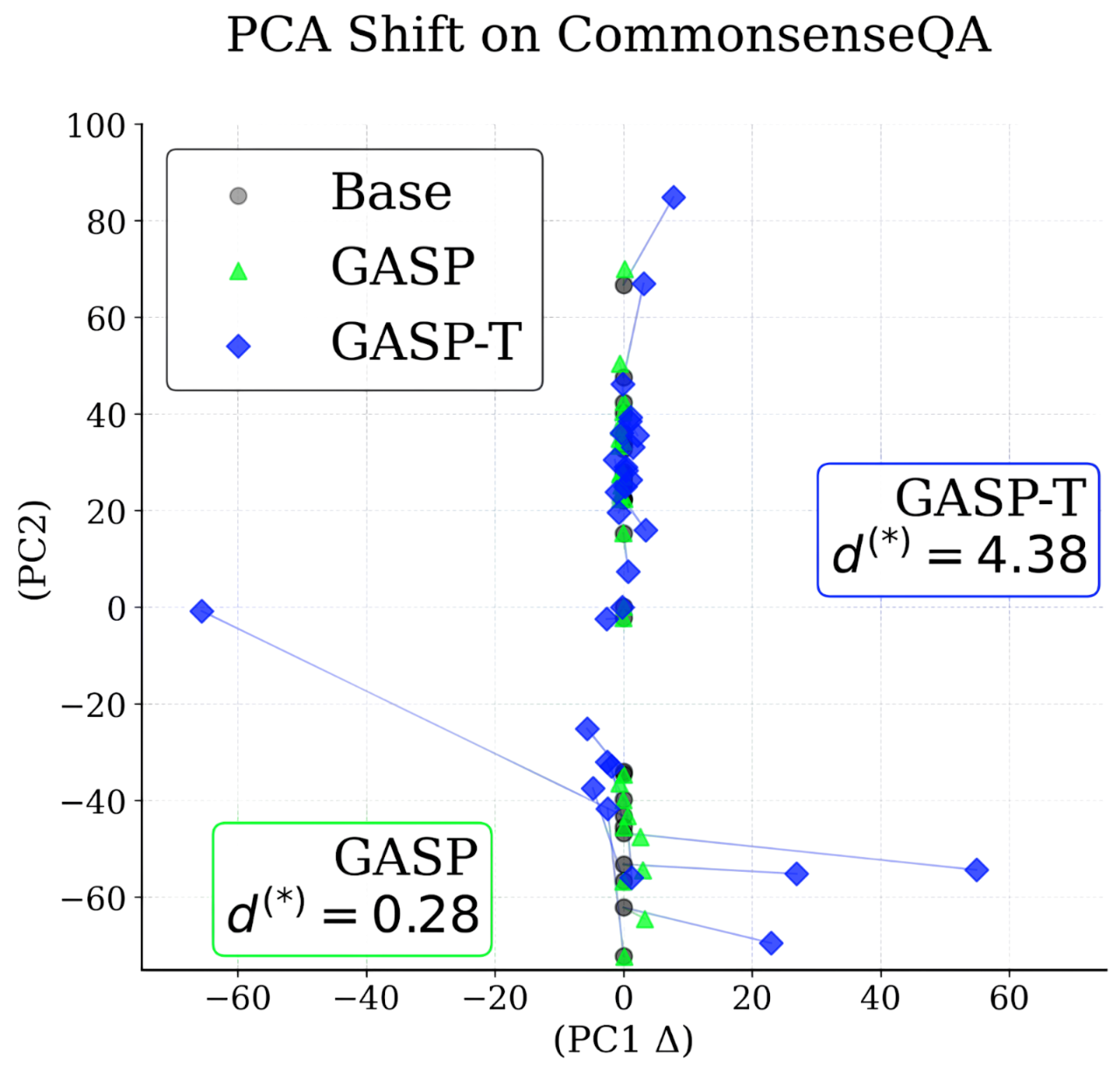}
  \vspace{-10pt}
  \caption{}
  \label{fig:mislead_maze}
\end{subfigure}\hfill
\begin{subfigure}[b]{0.24\linewidth}
  \centering
  \includegraphics[width=\linewidth]{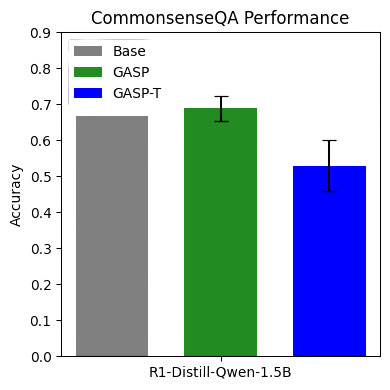}
  \vspace{-10pt}
  \caption{}
  \label{fig:clean_maze}
\end{subfigure}
%\vspace{-10pt}
\caption{\textbf{Representation drift correlates with QA drops.}
(a,c) show \emph{PCA shift} of layerwise mean hidden states on a fixed TruthfulQA / CommonsenseQA probe set. Each marker is a transformer layer $i$, plotted as $(\Delta m_{i,1},\, m_{i,2})$ after 2D PCA, where $\Delta m_{i,1}$ is the layer’s mean PC1 change from the base model (base lies on $\Delta\mathrm{PC1}{=}0$). The boxed $d^{(\ast)}$ is the Euclidean distance between base vs.\ post-training centroids in PCA space (global drift).
(b,d) show methods with larger drift (e.g., \textsc{GASP-T}) suffer larger accuracy drops, while low-drift \textsc{GASP} largely preserves accuracy.}
\label{fig:toycase}
\vspace{-20pt}
\end{figure*}
% \vspace{-5pt}

These observations mirror the core difficulty in off-trajectory recovery: when the agent is initialized in a misleading region (off the rail), the probability of reaching the goal, $p_\theta$, is initially very small. Under outcome-only GRPO with group size $G$, most update batches contain \emph{no} successful rollouts (all terminal rewards are zero), so group-relative normalization yields near-zero advantages and produces essentially no directed gradient toward the rare ``repair-then-finish'' behavior. Learning therefore reduces to blind exploration until an accidental recovery trajectory appears, which explains the delayed takeoff of GRPO-only in Figure~\ref{fig:mislead_maze}.
% \vspace{-4pt}

In-distribution repair guidance breaks this deadlock by explicitly upweighting short on-policy ``repair'' segments that rejoin the previously learned rail $\mathcal{R}$, thereby increasing the frequency of successful batches early in training. Once recoveries become common enough to appear within groups, GRPO can reliably propagate credit through the long-horizon dependency (repair $\rightarrow$ exploit prior skill $\rightarrow$ reach goal), while preserving the original clean-start behavior as reflected by the stable retention curve in Figure~\ref{fig:clean_maze}.
\begin{figure}[h]
\vspace{-15pt}
\begin{center}
\includegraphics[width=1\columnwidth]{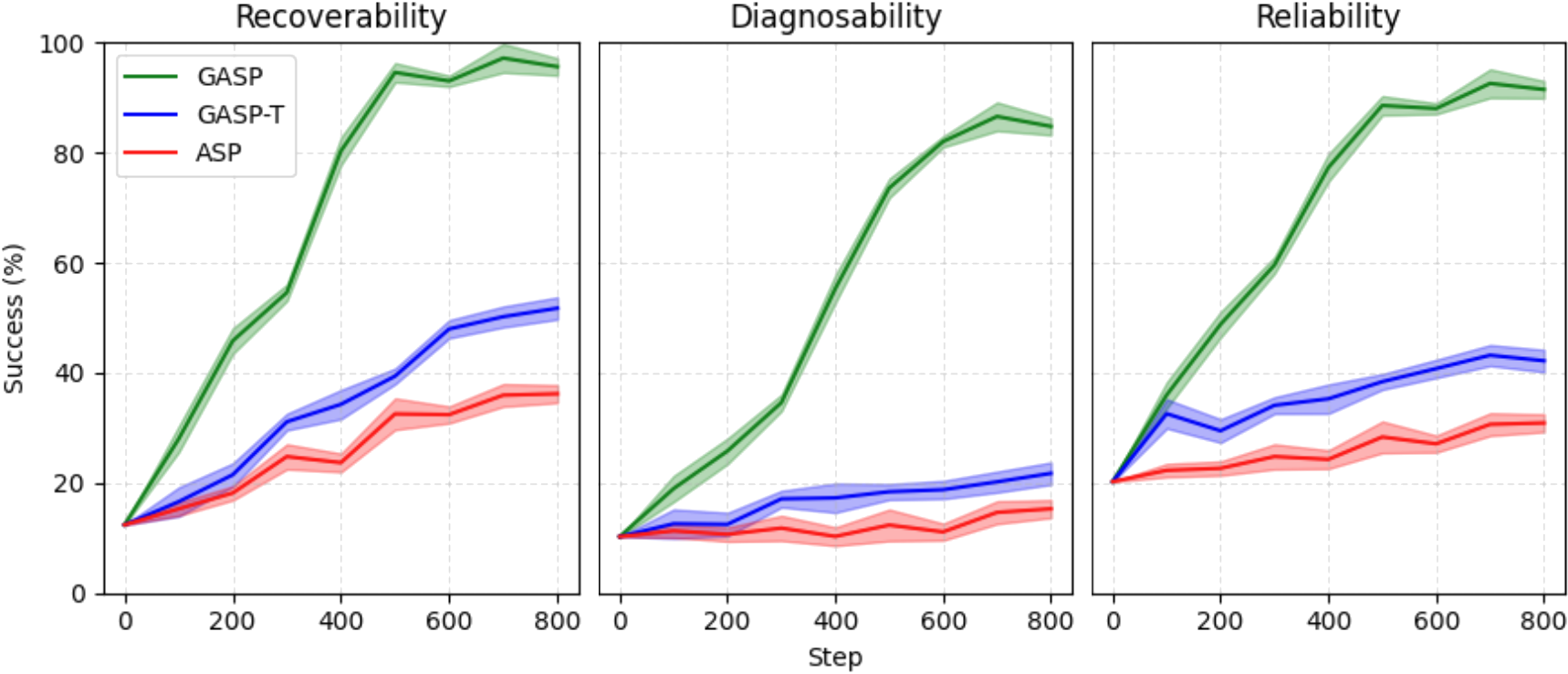}
\vspace{-20pt}
\caption{\textbf{Guidance improves learning efficiency.} Self-guided {\color{green}\textsc{GASP}} reaches substantially higher \emph{recoverability}, \emph{diagnosability}, and \emph{reliability} earlier than teacher-guided {\color{blue}\textsc{GASP-T}} and unguided self-play {\color{red}\textsc{ASP}}. Each reports the mean success rate averaged over the corresponding benchmark suite for that capability \S\ref{sec:eval_suite} at the given training step; shaded regions denote $\pm$1 std over 3 runs.}
\label{fig:efficiency}
\end{center}
\vspace{-20pt}
\end{figure}
% \vspace{-22pt}
\subsection{Efficiency and Representation Retention in LLM}
\label{sec:llm_eff_stab}
% \vspace{-8pt}
The maze analogue isolates \emph{why} recovery training is hard and what guidance can do. We now turn to the LLM setting to study \emph{how the source of guidance} shapes both learning efficiency and representational stability when training recovery behaviors on top of an already capable models.
% \vspace{-4pt}

Self-guided \textsc{GASP} improves MR-GSM8K substantially earlier and reaches a higher final score than \textsc{GASP}-T as shown in \cref{fig:efficiency}. This efficiency advantage coincides with a clear stability difference: \textsc{GASP}-T exhibits much larger \emph{layerwise representation shift}. Concretely, we compute $d^{(\ast)}$ by (i) collecting mean hidden-state representations on a fixed TruthfulQA probe set, (ii) projecting them onto a 2D PCA basis, and (iii) measuring the Euclidean distance between the pre- and post-training centroids in this PCA space. Larger $d^{(\ast)}$ therefore indicates a larger global movement of the model’s representations on the probe distribution.
% \vspace{-4pt}

This representation movement tracks downstream interference: the method with larger shift (teacher guided \textsc{GASP}) also suffers a larger drop in TruthfulQA accuracy, while self-guided \textsc{GASP} largely preserves it. Overall, in-distribution repair guidance is more continual-learning friendly in our setting: it achieves faster recovery learning while inducing smaller representation drift and better retention on an external QA benchmark.
\vspace{-12pt}
\section{Related Work}
\vspace{-6pt}
A growing line of work studies how models can \emph{improve their own reasoning} by generating feedback, revisions, or new training data without relying on additional human annotations \citep{zhou2025expo, lismart}. At inference time, iterative self-revision frameworks such as \textsc{Self-Refine} \citep{madaan2023self} and \textsc{Reflexion}\citep{shinn2023reflexion} treat the model as its own critic: the system produces an initial solution, generates natural-language feedback or reflections from task outcomes, and uses this feedback to revise subsequent attempts (often via an explicit memory buffer) without weight updates. At training time, self-bootstrapping methods such as \textsc{STaR} \citep{zelikman2022star} iteratively generate rationales, filter by final-answer correctness, and fine-tune on the resulting self-generated traces, while \textsc{SCoRe} \citep{kumar2024training} more directly targets \emph{correction} by training multi-turn self-correction policies with online RL and addressing distribution mismatch and collapse in offline self-correction fine-tuning. In parallel, self-play fine-tuning methods such as \textsc{SPIN} \citep{chen2024self} refine a model by iteratively generating synthetic responses and optimizing objectives that compare synthetic outputs to human-annotated data; follow-up work analyzes instability in pairwise self-play objectives and proposes more stable variants \citep{wang2025triplets}. Broader discussion of Reinforcement fine-tuning (RFT) at Appendix~\ref{app:self_improve}.

\vspace{-12pt}
\section{Conclusion}
\vspace{-5pt}
We introduced \textsc{GASP}, an adversarial self-play framework for robust reasoning under fallible conditioning using only verifiable outcome rewards. By coupling a polluter that learns creating corruptions with an agent that learns to detect-and-repair under the polluted conditioning, \textsc{GASP} turns strong-but-brittle reasoners into robust ones while typically preserving and often improving clean performance.

\clearpage
\section*{Impact Statement}
This paper presents work whose goal is to advance the robustness of machine learning models for reasoning under fallible or misleading context. By improving a model’s ability to detect and repair errors using verifiable feedback alone, this work may contribute to more reliable and trustworthy AI systems in practical deployments. While there are many potential societal consequences of improved reasoning robustness, we do not foresee specific negative impacts beyond those already well established for general advances in machine learning.

\bibliography{example_paper}
\bibliographystyle{icml2026}

\clearpage
% =========================
% Appendix (Copy-paste To-Do Skeleton)
% =========================
\onecolumn
\appendix
\section{Appendix}
\label{sec:appendix}

\subsection{Learning Signal Scarcity}
\label{app:learning_signal_scarcity}
{Formally, define the agent’s \emph{recovery rate} under the polluted steer as
\[
p_\theta \;\triangleq\; \Pr_{\tau_{\mathrm{off}}\sim \pi_\theta(\cdot\mid s^{\mathrm{poll}})}\!\big[a_{\mathrm{off}}=a^\star\big],
\]
i.e., the probability that an on-policy rollout conditioned on $s^{\mathrm{poll}}$ ends with the correct final answer. GRPO samples a group of $G$ rollouts for the same $(q,s^{\mathrm{poll}})$. Assuming these rollouts are approximately independent, the probability that the group contains \emph{at least one} successful recovery is
\[
\Pr(\text{any success in group}) \;=\; 1-(1-p_\theta)^G.
\]
In the rare-recovery regime where $Gp_\theta \ll 1$ (equivalently $p_\theta \ll 1/G$), this simplifies to
the intuitive approximation
\[
1-(1-p_\theta)^G \;\approx\; Gp_\theta,
\]
meaning the chance of seeing \emph{any} success in a group grows only linearly with both group size $G$ and
recovery rate $p_\theta$.

When $p_\theta$ is tiny (empirically typical under pollution), most groups contain no successful trajectories, so the terminal rewards are nearly identical across the group (often all zeros) and group-relative normalization produces uninformative advantages. As a result, the GRPO update provides little directed pressure toward recovery and is dominated by noise.}

\subsection{Limitations of Repair Guidance with a Teacher Model}
\label{app:ood_repair_guidance}
A straightforward approach is to use a strong external teacher policy $\pi_{\mathrm{T}}$ to generate a repair snippet $w^{\mathrm{fix}}_{\mathrm{T}} \sim \pi_{\mathrm{T}}(\cdot \mid p)$ from a diagnosis prompt $p$ that exposes the inconsistency (e.g., including both $w^{\mathrm{clean}}$ and $w^{\mathrm{poll}}$), and then behavior-clone this snippet under the \emph{deployment} conditioning $s^{\mathrm{poll}}$:
\[
J_{\mathrm{guide}}^{\mathrm{T}}(\theta)
=
\mathbb{E}\big[\log \pi_\theta(w^{\mathrm{fix}}_{\mathrm{T}} \mid s^{\mathrm{poll}})\big].
\]
The resulting agent objective is
\[
J_{\mathrm{agent}}(\theta)
=
J_{\mathrm{GRPO}}^{\mathrm{rec}}(\theta)
+
\lambda\, J_{\mathrm{guide}}^{\mathrm{T}}(\theta).
\]

Empirically, teacher fixes can be \emph{semantically correct yet inefficient as learning signals} because they often lie in low-probability regions of the student policy. Concretely, under the polluted steer,
\[
\pi_\theta(w^{\mathrm{fix}}_{\theta}\mid s^{\mathrm{poll}})\gg \pi_\theta(w^{\mathrm{fix}}_{\mathrm{T}}\mid s^{\mathrm{poll}})
\]
\vspace{-0.8cm}
\[
\quad \text{for typical on-policy continuations } w_\theta^{\mathrm{fix}} ,
\]
so the update must push probability mass into regions that the current policy rarely visits. Moreover, under an explicit KL penalty to a reference policy, increasing the likelihood of a repair $w^{\mathrm{fix}}_{\mathrm{T}}$ with $\pi_\theta(w^{\mathrm{fix}}_{\mathrm{T}}\mid s^{\mathrm{poll}})\approx 0$ requires a large deviation from the reference, so the KL term directly suppresses the update and slows learning.

\subsection{Why In-Distribution Guidance Yields Larger First-Order Gains}
\label{app:proof_guidance_scaling}
\paragraph{Setup.}
Fix a polluted steer (deployment conditioning) $s^{\mathrm{poll}}$.
Let the model generate a full continuation $o=(o_1,\dots,o_T)$ (e.g., repair text plus the remaining reasoning and final answer)
according to an autoregressive policy $\pi_\theta(o\mid s^{\mathrm{poll}})=\prod_{t=1}^T \pi_\theta(o_t\mid s^{\mathrm{poll}},o_{<t})$.
Let $r(o)\in[0,1]$ denote the terminal verifiable reward (e.g., correctness indicator).
Define the (conditional) deployment objective
\begin{equation}
J(\theta \mid s^{\mathrm{poll}})
\;:=\;
\mathbb{E}_{o\sim \pi_\theta(\cdot\mid s^{\mathrm{poll}})}\!\big[r(o)\big].
\label{eq:J_cond}
\end{equation}
(An unconditional objective that averages over $(q,a^\star)$ and corresponding $s^{\mathrm{poll}}$ follows by linearity of expectation.)

\paragraph{Guidance update.}
Consider an imitation-style guidance term that increases the likelihood of a specific repair snippet $w^{\mathrm{fix}}$
under the \emph{deployment} conditioning $s^{\mathrm{poll}}$:
\begin{equation}
L_{\mathrm{guide}}(\theta)
\;:=\;
\omega \,\log \pi_\theta\!\big(w^{\mathrm{fix}}\mid s^{\mathrm{poll}}\big),
\qquad \omega>0.
\label{eq:L_guide}
\end{equation}
A gradient-ascent step with step size $\eta>0$ is
\begin{equation}
\theta^+ \;=\; \theta + \eta\, g,
\qquad
g \;:=\; \nabla_\theta L_{\mathrm{guide}}(\theta)
= \omega \,\nabla_\theta \log \pi_\theta\!\big(w^{\mathrm{fix}}\mid s^{\mathrm{poll}}\big).
\label{eq:g_def}
\end{equation}

\paragraph{Goal.}
We prove that the \emph{first-order} improvement in $J(\theta\mid s^{\mathrm{poll}})$ due to the guidance step
contains a dominant term that scales linearly with $\pi_\theta(w^{\mathrm{fix}}\mid s^{\mathrm{poll}})$.
In particular, under mild near-orthogonality assumptions (standard in high-dimensional score-function analyses),
the leading contribution is proportional to
$\pi_\theta(w^{\mathrm{fix}}\mid s^{\mathrm{poll}})\big\|\nabla_\theta \log \pi_\theta(w^{\mathrm{fix}}\mid s^{\mathrm{poll}})\big\|_2^2$,
matching the scaling used in \S\ref{sec:guidance}.

\subsection{Main decomposition}

\begin{lemma}[First-order gain decomposition]
\label{lem:first_order_gain}
Let $\theta^+$ be defined by \eqref{eq:g_def}. Then a first-order Taylor expansion gives
\begin{equation}
\Delta J
\;:=\;
J(\theta^+ \mid s^{\mathrm{poll}}) - J(\theta \mid s^{\mathrm{poll}})
\;\approx\;
\eta\, \nabla_\theta J(\theta\mid s^{\mathrm{poll}})^\top g.
\label{eq:taylor}
\end{equation}
Moreover, writing $\pi_\theta(\cdot)$ as shorthand for $\pi_\theta(\cdot\mid s^{\mathrm{poll}})$,
\begin{align}
\nabla_\theta J(\theta\mid s^{\mathrm{poll}})
&=
\mathbb{E}_{o\sim\pi_\theta}\!\Big[r(o)\,\nabla_\theta \log \pi_\theta(o)\Big],
\label{eq:true_grad}
\\
\nabla_\theta J(\theta\mid s^{\mathrm{poll}})^\top g
&=
\omega\,
\mathbb{E}_{o\sim\pi_\theta}\!\Big[
r(o)\,
\big\langle \nabla_\theta \log \pi_\theta(o),\;
\nabla_\theta \log \pi_\theta(w^{\mathrm{fix}})\big\rangle
\Big].
\label{eq:inner_prod_expand}
\end{align}
Finally, letting $\Omega_{\mathrm{fix}}$ denote the set of continuations whose prefix equals $w^{\mathrm{fix}}$,
\[
\Omega_{\mathrm{fix}}
:=\{o:\; o_{1:|w^{\mathrm{fix}}|}=w^{\mathrm{fix}}\},
\]
the inner product decomposes exactly as
\begin{align}
\nabla_\theta J(\theta\mid s^{\mathrm{poll}})^\top g
&=
\omega\,\underbrace{
\sum_{o\in \Omega_{\mathrm{fix}}}
\pi_\theta(o)\,r(o)\,
\big\langle \nabla_\theta \log \pi_\theta(o),\;
\nabla_\theta \log \pi_\theta(w^{\mathrm{fix}})\big\rangle
}_{T_1}
\;+\;
\end{align}
\begin{align}
\omega\,\underbrace{
\sum_{o\notin \Omega_{\mathrm{fix}}}
\pi_\theta(o)\,r(o)\,
\big\langle \nabla_\theta \log \pi_\theta(o),\;
\nabla_\theta \log \pi_\theta(w^{\mathrm{fix}})\big\rangle
}_{T_2}.
\label{eq:T1T2}
\end{align}
\end{lemma}

\begin{proof}
Equation \eqref{eq:taylor} is the first-order Taylor approximation of $J(\cdot\mid s^{\mathrm{poll}})$ around $\theta$.
For \eqref{eq:true_grad}, apply the score-function identity:
\[
\nabla_\theta J(\theta\mid s^{\mathrm{poll}})
=
\nabla_\theta \sum_o \pi_\theta(o)\,r(o)
=
\sum_o \pi_\theta(o)\,r(o)\,\nabla_\theta \log \pi_\theta(o)
=
\mathbb{E}_{o\sim\pi_\theta}[r(o)\nabla_\theta \log \pi_\theta(o)].
\]
Substituting $g=\omega\nabla_\theta \log \pi_\theta(w^{\mathrm{fix}})$ yields \eqref{eq:inner_prod_expand}.
The split \eqref{eq:T1T2} is obtained by partitioning the support of $o$ into $\Omega_{\mathrm{fix}}$ and its complement.
\end{proof}

\subsection{Isolating the in-distribution scaling}

We now show that $T_1$ contains a term proportional to $\pi_\theta(w^{\mathrm{fix}}\mid s^{\mathrm{poll}})$.
For any $o\in \Omega_{\mathrm{fix}}$, write $o=w^{\mathrm{fix}}\circ z$, where $z$ is the suffix after the snippet.
Then
\begin{align}
\pi_\theta(o\mid s^{\mathrm{poll}})
&=
\pi_\theta(w^{\mathrm{fix}}\mid s^{\mathrm{poll}})
\;\pi_\theta(z\mid s^{\mathrm{poll}}, w^{\mathrm{fix}}),
\label{eq:factor_prob}
\\
\log \pi_\theta(o\mid s^{\mathrm{poll}})
&=
\log \pi_\theta(w^{\mathrm{fix}}\mid s^{\mathrm{poll}})
+
\log \pi_\theta(z\mid s^{\mathrm{poll}}, w^{\mathrm{fix}}),
\label{eq:factor_logprob}
\\
\nabla_\theta \log \pi_\theta(o\mid s^{\mathrm{poll}})
&=
\nabla_\theta \log \pi_\theta(w^{\mathrm{fix}}\mid s^{\mathrm{poll}})
+
\nabla_\theta \log \pi_\theta(z\mid s^{\mathrm{poll}}, w^{\mathrm{fix}}).
\label{eq:factor_grad}
\end{align}
Plugging \eqref{eq:factor_prob}--\eqref{eq:factor_grad} into $T_1$ from \eqref{eq:T1T2} gives
\begin{align}
T_1
&=
\pi_\theta(w^{\mathrm{fix}}\mid s^{\mathrm{poll}})
\;\mathbb{E}_{z\sim \pi_\theta(\cdot\mid s^{\mathrm{poll}},w^{\mathrm{fix}})}
\Big[
r(w^{\mathrm{fix}}\circ z)
\big\langle
\nabla_\theta \log \pi_\theta(w^{\mathrm{fix}}) + \nabla_\theta \log \pi_\theta(z\mid s^{\mathrm{poll}},w^{\mathrm{fix}}),
\;
\nabla_\theta \log \pi_\theta(w^{\mathrm{fix}})
\big\rangle
\Big]
\nonumber\\
&=
\pi_\theta(w^{\mathrm{fix}}\mid s^{\mathrm{poll}})
\;\mathbb{E}_{z\sim \pi_\theta(\cdot\mid s^{\mathrm{poll}},w^{\mathrm{fix}})}
\Big[
r(w^{\mathrm{fix}}\circ z)
\Big(
\big\|\nabla_\theta \log \pi_\theta(w^{\mathrm{fix}})\big\|_2^2
+
\big\langle
\nabla_\theta \log \pi_\theta(z\mid s^{\mathrm{poll}},w^{\mathrm{fix}}),
\;
\nabla_\theta \log \pi_\theta(w^{\mathrm{fix}})
\big\rangle
\Big)
\Big].
\label{eq:T1_expand}
\end{align}

Define the (conditional) success probability after emitting the repair snippet:
\begin{equation}
Q_\theta(s^{\mathrm{poll}}, w^{\mathrm{fix}})
\;:=\;
\mathbb{E}_{z\sim \pi_\theta(\cdot\mid s^{\mathrm{poll}},w^{\mathrm{fix}})}
\big[r(w^{\mathrm{fix}}\circ z)\big]
\in[0,1].
\label{eq:Q_def}
\end{equation}
Then \eqref{eq:T1_expand} can be written as
\begin{align}
T_1
&=
\pi_\theta(w^{\mathrm{fix}}\mid s^{\mathrm{poll}})
\;Q_\theta(s^{\mathrm{poll}}, w^{\mathrm{fix}})
\;\big\|\nabla_\theta \log \pi_\theta(w^{\mathrm{fix}}\mid s^{\mathrm{poll}})\big\|_2^2
\;+\;
\end{align}
\begin{align}
\underbrace{
\pi_\theta(w^{\mathrm{fix}}\mid s^{\mathrm{poll}})
\;\mathbb{E}_{z} \!\Big[
r(w^{\mathrm{fix}}\circ z)\,
\big\langle
\nabla_\theta \log \pi_\theta(z\mid s^{\mathrm{poll}},w^{\mathrm{fix}}),\;
\nabla_\theta \log \pi_\theta(w^{\mathrm{fix}})\big\rangle
\Big]
}_{E_{\mathrm{suffix}}}.
\label{eq:T1_Q_plus_err}
\end{align}

\paragraph{Near-orthogonality assumption (standard).}
The remaining terms $E_{\mathrm{suffix}}$ and $T_2$ are cross inner-products between score functions of \emph{different}
token sequences (or different segments of the same sequence). In very high-dimensional parameter spaces, these
cross terms are typically small unless the sequences are extremely similar; see, e.g., the orthogonality-based analysis
in Appendix A of ExPO.
Formally, we assume the following bounded-cross-term condition.

\begin{assumption}[Small cross score inner-products]
\label{asmp:cross_small}
There exist nonnegative constants $\epsilon_{\mathrm{suffix}}$ and $\epsilon_{\mathrm{off}}$ such that
\begin{align}
\Big|
\mathbb{E}_{z\sim \pi_\theta(\cdot\mid s^{\mathrm{poll}},w^{\mathrm{fix}})}
\!\Big[
r(w^{\mathrm{fix}}\circ z)\,
\big\langle
\nabla_\theta \log \pi_\theta(z\mid s^{\mathrm{poll}},w^{\mathrm{fix}}),\;
\nabla_\theta \log \pi_\theta(w^{\mathrm{fix}}\mid s^{\mathrm{poll}})
\big\rangle
\Big]
\Big|
&\le \epsilon_{\mathrm{suffix}},
\label{eq:eps_suffix}
\\
\Big|
\sum_{o\notin \Omega_{\mathrm{fix}}}
\pi_\theta(o\mid s^{\mathrm{poll}})\,r(o)\,
\big\langle \nabla_\theta \log \pi_\theta(o\mid s^{\mathrm{poll}}),\;
\nabla_\theta \log \pi_\theta(w^{\mathrm{fix}}\mid s^{\mathrm{poll}})\big\rangle
\Big|
&\le \epsilon_{\mathrm{off}}.
\label{eq:eps_off}
\end{align}
\end{assumption}

\begin{proposition}[Dominant scaling with $\pi_\theta(w^{\mathrm{fix}}\mid s^{\mathrm{poll}})$]
\label{prop:dominant_scaling}
Under Assumption~\ref{asmp:cross_small}, the first-order improvement satisfies
\begin{align}
\Delta J
\;\approx\;
\eta\,\omega\,
\pi_\theta(w^{\mathrm{fix}}\mid s^{\mathrm{poll}})
\;Q_\theta(s^{\mathrm{poll}}, w^{\mathrm{fix}})
\;\big\|\nabla_\theta \log \pi_\theta(w^{\mathrm{fix}}\mid s^{\mathrm{poll}})\big\|_2^2
\;+\;
\eta\,\omega\,\xi,
\label{eq:deltaJ_main}
\end{align}
where the residual $\xi$ is bounded as $|\xi|\le \pi_\theta(w^{\mathrm{fix}}\mid s^{\mathrm{poll}})\epsilon_{\mathrm{suffix}}+\epsilon_{\mathrm{off}}$.
In particular, when cross terms are small (small $\epsilon_{\mathrm{suffix}},\epsilon_{\mathrm{off}}$), the dominant contribution
scales linearly with $\pi_\theta(w^{\mathrm{fix}}\mid s^{\mathrm{poll}})$.
\end{proposition}

\begin{proof}
Combine Lemma~\ref{lem:first_order_gain} with \eqref{eq:T1_Q_plus_err}:
\[
\nabla_\theta J(\theta\mid s^{\mathrm{poll}})^\top g
=
\omega\,(T_1+T_2)
=
\omega\Big(
\pi_\theta(w^{\mathrm{fix}}\mid s^{\mathrm{poll}})
Q_\theta \|\nabla_\theta \log \pi_\theta(w^{\mathrm{fix}})\|_2^2
+ E_{\mathrm{suffix}}
+ T_2
\Big).
\]
Apply Assumption~\ref{asmp:cross_small} to bound $E_{\mathrm{suffix}}$ and $T_2$, and multiply by $\eta$ per \eqref{eq:taylor}.
\end{proof}

\paragraph{Recovering the simplified expression used in \S\ref{sec:guidance}.}
If $w^{\mathrm{fix}}$ is a near-certain repair under the current policy (i.e., $Q_\theta(s^{\mathrm{poll}},w^{\mathrm{fix}})\approx 1$),
and cross terms are negligible (so $\xi\approx 0$), Proposition~\ref{prop:dominant_scaling} reduces to
\begin{equation}
\Delta J
\;\approx\;
\eta\,\omega\,
\pi_\theta(w^{\mathrm{fix}}\mid s^{\mathrm{poll}})
\;\big\|\nabla_\theta \log \pi_\theta(w^{\mathrm{fix}}\mid s^{\mathrm{poll}})\big\|_2^2,
\label{eq:deltaJ_simplified}
\end{equation}
which is the scaling statement referenced in \S\ref{sec:guidance}.

\subsection{Corollary: why in-distribution repairs yield larger gains}

\begin{corollary}[Preference for in-distribution repairs]
\label{cor:in_dist_better}
Let $w^{(1)}$ and $w^{(2)}$ be two candidate repair snippets under the same $s^{\mathrm{poll}}$.
Assume they have comparable weights and score norms, and comparable downstream success probabilities:
\[
\omega^{(1)}\approx \omega^{(2)},\qquad
\big\|\nabla_\theta \log \pi_\theta(w^{(1)}\mid s^{\mathrm{poll}})\big\|_2^2 \approx
\big\|\nabla_\theta \log \pi_\theta(w^{(2)}\mid s^{\mathrm{poll}})\big\|_2^2,\qquad
Q_\theta(s^{\mathrm{poll}},w^{(1)})\approx Q_\theta(s^{\mathrm{poll}},w^{(2)}),
\]
and cross terms are negligible. Then the ratio of first-order gains satisfies
\begin{equation}
\frac{\Delta J(w^{(1)})}{\Delta J(w^{(2)})}
\;\approx\;
\frac{\pi_\theta(w^{(1)}\mid s^{\mathrm{poll}})}{\pi_\theta(w^{(2)}\mid s^{\mathrm{poll}})}.
\label{eq:gain_ratio}
\end{equation}
In particular, if $w^{(1)}$ is substantially more likely under the current policy than $w^{(2)}$ (more in-distribution),
then it yields a substantially larger expected first-order improvement for the same-sized update.
\end{corollary}

\begin{proof}
Apply Proposition~\ref{prop:dominant_scaling} to each snippet and cancel the approximately equal factors.
\end{proof}

\subsection{Why Few Successes Yield Noisy GRPO Updates}
\label{app:few-successes-noisy}
We formalize the claim that when successful trajectories are rare, GRPO's outcome-only update is \emph{noisy} because
the success-conditioned contribution is estimated from very few samples (often one), yielding a low signal-to-noise ratio.

\paragraph{Setup.}
Fix a conditioning context $s$ (e.g., a polluted steer $s_{\mathrm{poll}}$ in Sec.~4.4).
GRPO samples a group of $G$ rollouts $\{\tau_i\}_{i=1}^G$ from a behavior policy $\pi_{\theta_{\mathrm{old}}}$.
Each rollout $\tau_i$ has token sequence $o_i=(o_{i,1},\dots,o_{i,|o_i|})$ and terminal reward
$R_i \in \{0,1\}$ (verifiable correctness).
GRPO forms the group-normalized reward
\begin{equation}
\tilde R_i
\;:=\;
\frac{R_i - \overline R}{s_R + \varepsilon},
\qquad
\overline R := \frac{1}{G}\sum_{j=1}^G R_j,
\qquad
s_R := \sqrt{\frac{1}{G}\sum_{j=1}^G (R_j-\overline R)^2},
\label{eq:grpo-adv}
\end{equation}
and assigns $\hat A_{i,t}=\tilde R_i$ for all tokens $t$ (Sec.~3.2).

To isolate the variance mechanism, we analyze the \emph{unclipped} surrogate (equivalently, the first-order/small-step
regime where $\rho_{i,t}(\theta)\approx 1$ and the $\min/\mathrm{clip}$ does not activate). Define the per-trajectory
score-gradient average
\begin{equation}
S_i(\theta)
\;:=\;
\frac{1}{|o_i|}\sum_{t=1}^{|o_i|}
\nabla_\theta \log \pi_\theta(o_{i,t}\mid s, o_{i,<t}),
\label{eq:Si-def}
\end{equation}
so that the corresponding (unclipped) GRPO gradient estimator can be written as
\begin{equation}
\hat g(\theta)
\;:=\;
\frac{1}{G}\sum_{i=1}^G \tilde R_i\, S_i(\theta).
\label{eq:g-hat}
\end{equation}
(Clipping and KL regularization can only \emph{downweight} some terms and thus do not remove the scarcity-induced noise;
Remark~\ref{rem:clip} discusses this.)

\paragraph{Closed form of group-normalized rewards for binary outcomes.}
Let $K := \sum_{i=1}^G R_i$ be the number of successes in the group.

\begin{lemma}[Binary-reward GRPO advantages]
\label{lem:binary-adv}
Condition on $K=k$ with $1\le k \le G-1$. Then $\overline R = k/G$ and $s_R = \sqrt{k(G-k)}/G$.
Consequently,
\begin{equation}
\tilde R_i
=
\begin{cases}
a_k := \sqrt{\frac{G-k}{k}}, & \text{if } R_i=1,\\[6pt]
b_k := -\sqrt{\frac{k}{G-k}}, & \text{if } R_i=0,
\end{cases}
\label{eq:ak-bk}
\end{equation}
(up to the negligible $\varepsilon$ in~\eqref{eq:grpo-adv}).
If $k=0$ (all failures) or $k=G$ (all successes), then $s_R=0$ and $\tilde R_i \approx 0$ for all $i$ due to $\varepsilon$,
so $\hat g(\theta)\approx 0$.
\end{lemma}

\begin{proof}
When $K=k$, the reward multiset contains $k$ ones and $G-k$ zeros, so $\overline R=k/G$.
The variance is
\[
s_R^2
=
\frac{1}{G}\Big(k(1-k/G)^2 + (G-k)(0-k/G)^2\Big)
=
\frac{k(G-k)}{G^2},
\]
hence $s_R=\sqrt{k(G-k)}/G$. Plugging into~\eqref{eq:grpo-adv} yields~\eqref{eq:ak-bk}.
For $k\in\{0,G\}$ we have $s_R=0$ and $\tilde R_i=(R_i-\overline R)/\varepsilon\approx 0$.
\end{proof}

\paragraph{Decomposition as a difference of two sample means.}
Let $S^{(1)}:=\{i:R_i=1\}$ and $S^{(0)}:=\{i:R_i=0\}$.
Define the within-group sample means
\begin{equation}
\bar S_1 := \frac{1}{k}\sum_{i\in S^{(1)}} S_i(\theta),
\qquad
\bar S_0 := \frac{1}{G-k}\sum_{i\in S^{(0)}} S_i(\theta).
\label{eq:barS}
\end{equation}

\begin{lemma}[Gradient estimator given $K=k$]
\label{lem:grad-decomp}
Condition on $K=k$ with $1\le k\le G-1$. Then
\begin{equation}
\hat g(\theta)
=
c_k\big(\bar S_1 - \bar S_0\big),
\qquad
c_k := \frac{\sqrt{k(G-k)}}{G}.
\label{eq:ghat-decomp}
\end{equation}
In particular, when $k=1$, the update depends on a \emph{single} success trajectory via $\bar S_1=S_{i^\star}$.
\end{lemma}

\begin{proof}
By Lemma~\ref{lem:binary-adv}, $\tilde R_i=a_k$ on successes and $\tilde R_i=b_k$ on failures. Thus
\[
\hat g(\theta)
=
\frac{1}{G}\Big(a_k\sum_{i\in S^{(1)}} S_i(\theta) + b_k\sum_{i\in S^{(0)}} S_i(\theta)\Big)
=
\frac{1}{G}\Big(a_k k\,\bar S_1 + b_k (G-k)\,\bar S_0\Big).
\]
Using $a_k k = \sqrt{k(G-k)}$ and $b_k(G-k)=-\sqrt{k(G-k)}$ yields~\eqref{eq:ghat-decomp}.
\end{proof}

\paragraph{Variance and signal-to-noise.}
Assume that conditional on the reward, the score-gradients have finite second moments:
\begin{equation}
\mu_r := \mathbb{E}[S_i(\theta)\mid R_i=r],
\qquad
\Sigma_r := \mathrm{Cov}(S_i(\theta)\mid R_i=r),
\qquad r\in\{0,1\}.
\label{eq:muSigma}
\end{equation}
Assume also that $\{S_i(\theta)\}_{i=1}^G$ are conditionally independent given $\{R_i\}_{i=1}^G$
(standard for independent rollouts).

\begin{proposition}[Conditional covariance and SNR improve with \#successes]
\label{prop:var-snr}
Condition on $K=k$ with $1\le k\le G-1$.
Then
\begin{align}
\mathbb{E}\big[\hat g(\theta)\mid K=k\big]
&= c_k\big(\mu_1-\mu_0\big),
\label{eq:cond-mean}\\
\mathrm{Cov}\big(\hat g(\theta)\mid K=k\big)
&= c_k^2\left(\frac{\Sigma_1}{k} + \frac{\Sigma_0}{G-k}\right).
\label{eq:cond-cov}
\end{align}
Moreover, defining a (squared) signal-to-noise ratio using Frobenius/trace variance,
\begin{equation}
\mathrm{SNR}^2(k)
\;:=\;
\frac{\big\|\mathbb{E}[\hat g(\theta)\mid K=k]\big\|_2^2}{
\mathbb{E}\big[\|\hat g(\theta)-\mathbb{E}[\hat g(\theta)\mid K=k]\|_2^2 \mid K=k\big]},
\label{eq:snr-def}
\end{equation}
we obtain the closed form
\begin{equation}
\mathrm{SNR}^2(k)
=
\frac{\|\mu_1-\mu_0\|_2^2}{
\mathrm{tr}(\Sigma_1)/k + \mathrm{tr}(\Sigma_0)/(G-k)}.
\label{eq:snr-closed}
\end{equation}
In particular, $\mathrm{SNR}(k)$ is increasing in $k$ (holding other terms fixed), and when $k=1$ the success term
$\mathrm{tr}(\Sigma_1)/k$ receives \emph{no averaging reduction}, reflecting that the success-conditioned contribution
is estimated from a single trajectory.
\end{proposition}

\begin{proof}
From Lemma~\ref{lem:grad-decomp}, $\hat g(\theta)=c_k(\bar S_1-\bar S_0)$.
By conditional independence and~\eqref{eq:muSigma},
$\mathbb{E}[\bar S_1\mid K=k]=\mu_1$ and $\mathbb{E}[\bar S_0\mid K=k]=\mu_0$,
yielding~\eqref{eq:cond-mean}.
Similarly, $\mathrm{Cov}(\bar S_1\mid K=k)=\Sigma_1/k$ and $\mathrm{Cov}(\bar S_0\mid K=k)=\Sigma_0/(G-k)$,
and the cross-covariance is zero, so
$\mathrm{Cov}(\bar S_1-\bar S_0\mid K=k)=\Sigma_1/k + \Sigma_0/(G-k)$.
Multiplying by $c_k^2$ gives~\eqref{eq:cond-cov}.
Finally,
\[
\mathbb{E}\big[\|\hat g-\mathbb{E}[\hat g\mid K=k]\|_2^2\mid K=k\big]
=
\mathrm{tr}\big(\mathrm{Cov}(\hat g\mid K=k)\big)
=
c_k^2\left(\frac{\mathrm{tr}(\Sigma_1)}{k}+\frac{\mathrm{tr}(\Sigma_0)}{G-k}\right),
\]
and
$\|\mathbb{E}[\hat g\mid K=k]\|_2^2=c_k^2\|\mu_1-\mu_0\|_2^2$.
Canceling $c_k^2$ yields~\eqref{eq:snr-closed}. Monotonicity in $k$ follows because $\mathrm{tr}(\Sigma_1)/k$
decreases with $k$.
\end{proof}

\paragraph{Interpretation (``few successes $\Rightarrow$ noisy'').}
Equation~\eqref{eq:cond-cov} shows that the success-side uncertainty scales as $\Sigma_1/k$.
Thus when $k$ is small (especially $k=1$), the success-conditioned mean gradient $\mu_1$ is estimated with high variance,
and~\eqref{eq:snr-closed} shows the resulting GRPO update has low SNR.
When $k=0$, Lemma~\ref{lem:binary-adv} implies $\hat g(\theta)\approx 0$, yielding an uninformative update.

\begin{remark}[Clipping/KL do not eliminate scarcity noise]
\label{rem:clip}
The analysis above uses the unclipped surrogate for clarity.
In GRPO, clipping replaces each token term by a quantity whose magnitude is \emph{upper-bounded} by the unclipped term.
Therefore, clipping/KL can further \emph{reduce} the signal magnitude in rare-success batches, but cannot increase the number
of successes $k$ nor remove the $1/k$ sampling noise in estimating $\mu_1$.
\end{remark}

% -------------------------
\subsection{Additional Related Works}
\label{app:self_improve}
Reinforcement fine-tuning (RFT) has become a dominant paradigm for scaling reasoning in large language models. Early systems largely followed \textsc{PPO}-style RLHF with learned reward models \citep{ouyang2022training}. 
More recently, RLVR and critic-free objectives such as \textsc{GRPO} \citep{shao2024deepseekmath} showed that outcome-based supervision—often only final-answer correctness—can elicit strong long chain-of-thought reasoning while simplifying the RL stack. This has triggered a wave of work refining GRPO to improve learning efficiency, stability, and bias properties. For instance, \textsc{DAPO} \citep{yu2025dapo} adjusts group-relative advantage estimation and relaxes KL constraints to enable faster departure from the base model; \textsc{GMPO} \citep{zhao2025geometric} replaces GRPO’s arithmetic aggregation with a geometric mean to improve stability; \textsc{Dr.GRPO} \citep{liu2025understanding} analyzes sources of bias and proposes variants that remove length and standard-deviation normalization to better match an unbiased policy-gradient objective; and \textsc{GSPO} \citep{zheng2025group} moves from token-level to sequence-level likelihood ratios and clipping, aligning optimization more directly with sequence-level outcomes. Collectively, these methods sharpen the optimization of \emph{clean} reasoning performance: given a reliable prompt and a well-specified objective, they aim to make outcome-RL updates more effective and less destabilizing. 
% However, this line of work largely treats the \emph{conditioning context} as trustworthy—i.e., the prompt, intermediate trajectory, or provided solution is assumed to be correct and non-adversarial. 
In deployed settings, this assumption of a trustworthy conditioning context often fails: users make mistakes, partial solutions may contain errors, and intermediate context can be noisy, misleading, or distribution-shifted. 
As a result, a model can be a strong \emph{solo} reasoner yet remain brittle under fallible context, where the key capability is not merely to “solve,” but to \textbf{detect issues, repair the context, and recover the correct answer}. 

% -------------------------
\subsection{Full Benchmarks}
\label{app:full_benchmark_eval}
\label{app:baselines}
\begin{table*}[h]
\centering
\footnotesize
\setlength{\tabcolsep}{2.6pt}
\renewcommand{\arraystretch}{0.82}

\begin{tabular}{ll c c c c c c}
\toprule
& &
\multicolumn{1}{c}{\textbf{Recoverability}} &
\multicolumn{2}{c}{\textbf{Diagnosability}} &
\multicolumn{2}{c}{\textbf{Reliability}} &
\multicolumn{1}{c}{\textbf{Avg}} \\
\cmidrule(lr){3-3}
\cmidrule(lr){4-5}
\cmidrule(lr){6-7}
\cmidrule(lr){8-8}

\textbf{Model} & \textbf{Method} &
{GSM8K} &
{MR-GSM8K} & {MHPP} &
{GSM8K} & {CQA} &
{Avg} \\
\midrule

\multirow{6}{*}{R1-Distill-Qwen-1.5B}
& Initial ckpt              & \score{9.8}  & \score{4.5}  & \score{2.9}  & \score{33.3} & \score{37.1} & 17.5 \\
& \textsc{ASP}              & \meanstd{26.3}{1.8} & \meanstd{11.2}{0.7} & \meanstd{5.2}{0.4}  & \meanstd{44.4}{1.9} & \meanstd{45.3}{1.8} & 26.5 \\
& \textsc{GASP-p-fixed}     & \meanstd{30.8}{1.6} & \meanstd{13.0}{1.0} & \meanstd{6.0}{0.5}  & \meanstd{52.0}{2.0} & \meanstd{53.5}{1.9} & 31.1 \\
& \textsc{GASP-T}           & \meanstd{35.5}{1.5} & \meanstd{15.1}{1.2} & \meanstd{7.0}{0.6}  & \meanstd{59.9}{2.1} & \meanstd{61.2}{2.0} & 35.7 \\
& \textsc{GASP-p-gpt-5}      & \meanstd{74.2}{1.3} & \meanstd{40.8}{1.9} & \meanstd{15.4}{1.1} & \meanstd{63.9}{2.2} & \meanstd{64.0}{2.0} & 51.7 \\
\rowcolor{GASPshade}\cellcolor{white}\strut
& \textsc{GASP}             & \meanstd{78.5}{1.2} & \meanstd{43.7}{2.1} & \meanstd{16.9}{1.2} & \meanstd{64.8}{2.3} & \meanstd{64.7}{2.1} & \textbf{53.7} \\
\midrule

\multirow{6}{*}{DeepScaleR-1.5B}
& Initial ckpt              & \score{10.2} & \score{6.2}  & \score{3.3}  & \score{39.2} & \score{40.8} & 19.9 \\
& \textsc{ASP}              & \meanstd{19.6}{1.5} & \meanstd{12.4}{0.8} & \meanstd{7.1}{0.6}  & \meanstd{45.4}{2.0} & \meanstd{48.2}{2.1} & 26.5 \\
& \textsc{GASP-p-fixed}     & \meanstd{23.1}{1.2} & \meanstd{14.6}{1.1} & \meanstd{8.4}{0.7}  & \meanstd{53.2}{2.1} & \meanstd{56.5}{2.0} & 31.2 \\
& \textsc{GASP-T}           & \meanstd{26.5}{1.3} & \meanstd{16.7}{1.4} & \meanstd{9.6}{0.9}  & \meanstd{61.3}{2.2} & \meanstd{65.1}{2.2} & 35.8 \\
& \textsc{GASP-p-gpt-5}      & \meanstd{75.0}{1.2} & \meanstd{43.2}{2.0} & \meanstd{17.0}{1.3} & \meanstd{66.6}{2.3} & \meanstd{67.5}{2.1} & 53.9 \\
\rowcolor{GASPshade}\cellcolor{white}\strut
& \textsc{GASP}             & \meanstd{79.7}{1.1} & \meanstd{46.8}{2.3} & \meanstd{18.6}{1.4} & \meanstd{67.3}{2.4} & \meanstd{68.3}{2.3} & \textbf{56.1} \\
\midrule

\multirow{6}{*}{Qwen3-4B}
& Initial ckpt              & \score{8.9}  & \score{10.7} & \score{5.1}  & \score{46.5} & \score{48.7} & 24.0 \\
& \textsc{ASP}              & \meanstd{28.5}{2.0} & \meanstd{19.2}{1.3} & \meanstd{9.6}{0.8}  & \meanstd{56.7}{2.2} & \meanstd{68.9}{2.5} & 36.6 \\
& \textsc{GASP-p-fixed}     & \meanstd{33.2}{1.6} & \meanstd{22.4}{1.4} & \meanstd{11.2}{1.0} & \meanstd{66.8}{2.1} & \meanstd{73.0}{2.0} & 41.3 \\
& \textsc{GASP-T}           & \meanstd{38.5}{1.5} & \meanstd{25.9}{1.7} & \meanstd{13.0}{1.1} & \meanstd{76.5}{2.0} & \meanstd{78.6}{2.3} & 46.5 \\
& \textsc{GASP-p-gpt-5}      & \meanstd{78.0}{1.3} & \meanstd{56.0}{2.2} & \meanstd{27.5}{1.7} & \meanstd{86.3}{1.9} & \meanstd{79.2}{2.1} & 65.4 \\
\rowcolor{GASPshade}\cellcolor{white}\strut
& \textsc{GASP}             & \meanstd{81.2}{0.9} & \meanstd{59.3}{2.5} & \meanstd{29.8}{1.9} & \meanstd{88.7}{1.8} & \meanstd{79.7}{2.4} & \textbf{67.7} \\
\midrule

\multirow{6}{*}{Qwen3-8B (lora)}
& Initial ckpt              & \score{7.2}  & \score{16.2} & \score{8.2}  & \score{53.6} & \score{69.9} & 31.0 \\
& \textsc{ASP}              & \meanstd{34.4}{2.3} & \meanstd{28.7}{1.9} & \meanstd{14.4}{1.2} & \meanstd{69.9}{2.6} & \meanstd{85.4}{2.1} & 46.6 \\
& \textsc{GASP-p-fixed}     & \meanstd{40.2}{1.7} & \meanstd{33.0}{1.6} & \meanstd{16.6}{1.3} & \meanstd{79.8}{2.0} & \meanstd{90.5}{1.6} & 52.0 \\
& \textsc{GASP-T}           & \meanstd{46.4}{1.6} & \meanstd{38.7}{2.0} & \meanstd{19.4}{1.7} & \meanstd{89.0}{1.9} & \meanstd{94.0}{1.5} & 57.5 \\
& \textsc{GASP-p-gpt-5}      & \meanstd{83.0}{0.8} & \meanstd{68.5}{2.0} & \meanstd{50.5}{2.4} & \meanstd{89.5}{1.6} & \meanstd{94.5}{1.3} & 77.2 \\
\rowcolor{GASPshade}\cellcolor{white}\strut
& \textsc{GASP}             & \meanstd{85.7}{0.7} & \meanstd{72.7}{2.1} & \meanstd{54.1}{2.8} & \meanstd{89.8}{1.5} & \meanstd{94.8}{1.2} & \textbf{79.4} \\
\bottomrule
\end{tabular}

\vspace{2pt}
\caption{\textbf{Main robustness results (pass@1).} Mean $\pm$ std over 3 training runs for \textsc{ASP}/\textsc{GASP-p-fixed}/\textsc{GASP-T}/\textsc{GASP-p-gpt5}/\textsc{GASP}; \textbf{Initial ckpt} is a single evaluation. \textbf{Recoverability}: GSM8K accuracy under corrupted context (held-out GPT-5 polluter). \textbf{Diagnosability}: \textbf{ACCreason} on MR-GSM8K and MHPP \citep{zeng2023mr,zeng2024mr}. \textbf{Reliability}: accuracy on perturbed GSM8K and CommonsenseQA from RUPBench \citep{wang2024rupbench}. \textbf{Avg}: unweighted mean of the five metrics. \textsc{GASP-T} uses teacher-generated (GPT-5) repair guidance in place of self-guided in-distribution repairs. \textsc{GASP-p-fixed} fixes the polluter (no adversarial polluter updates) during training. \textsc{GASP-p-gpt-5} replaces the learned polluter with GPT-5 during training.}
\label{tab:main}
\vspace{-25pt}
\end{table*}

\section{Prompt Templates and Role Conditioning}
\label{app:prompts}

\lstdefinestyle{promptcode}{
  basicstyle=\ttfamily\small,
  columns=fullflexible,
  keepspaces=true,
  breaklines=true,
  breakatwhitespace=false,
  frame=single,
  rulecolor=\color{black!20},
  showstringspaces=false,
  upquote=true,
  tabsize=2,
  backgroundcolor=\color{blue!6} % <-- light blue background
}

\begin{itemize}[leftmargin=*]
    \item \textbf{Agent prompt (system prompt only).}
\begin{lstlisting}[style=promptcode]
[SYSTEM]
    A conversation between User and Assistant. The user asks a question, and the Assistant solves it.
    The assistant first thinks about the reasoning process in the mind and then provides the user with the answer.
    The reasoning process and answer are enclosed within <think> </think> and <answer> </answer> tags, respectively, i.e.,
    <think> reasoning process here </think><answer> answer here </answer>
\end{lstlisting}
    \item \textbf{Polluter prompt.}
\begin{lstlisting}[style=promptcode]
[SYSTEM]
You are a helpful assistant.

[USER]
You will be given:
  - a math problem,
  - the student's reasoning prefix (what they have already written),
  - and a SHORT CONTINUATION TEXT that the student originally wrote.

Your job is to EDIT ONLY that short continuation text so that:
  - it stays natural, coherent, and in the same style and level of detail,
  - it contains at least one explicit mathematical statement (numbers / formula / relation),
  - and at least one of those mathematical statements becomes objectively incorrect.

You must treat the original continuation as a piece of text to be edited IN PLACE, not as a topic to write a new paragraph about.

EDIT PRIORITY:
  1) If the original continuation already contains any numbers, numeric expressions, indices,
     or concrete mathematical facts, you MUST introduce the error by MODIFYING those existing
     numbers or facts (for example, change 21st to 25th; change 1, 10, 19 to 1, 16, 29;
     change 10 - 1 = 9 to 10 - 1 = 8; change (n-1)d to (n+1)d).
     In this case you are NOT allowed to introduce entirely new sentences, and you should avoid
     inventing new numbers that did not appear before, except by locally changing some digits.
  2) Only if the original continuation contains NO numerical or mathematical content at all
     may you add a new mathematical statement and make that statement false.

CRITICAL CONSTRAINT ON EQUATIONS:
  - If the text contains any equations of the form 'a - b = c', 'a + b = c', or 'a * b = c',
    the rewritten text MUST contain at least one such equation that is mathematically FALSE.
  - You must NOT leave all such equations mathematically correct after you edit the numbers.
  - For example, if you change '10 - 1 = 9', you must ensure the new equality is wrong, such as
    '16 - 1 = 9' or '10 - 1 = 8'.

OUTPUT RULES:
  - The output must be a lightly edited version of the original continuation, not a new paragraph.
  - Keep the same number of sentences if possible.
  - Do NOT state any final numeric answer.
  - Do NOT say that anything is wrong, misleading, or experimental.
  - Output ONLY the edited continuation, wrapped exactly in <polluted>...</polluted> tags.
  - No extra text before or after the <polluted> block.

[PROBLEM]
{q}

[STUDENT REASONING PREFIX]
{prefix}

[ORIGINAL SHORT CONTINUATION TO EDIT]
{continuation}

[YOUR OUTPUT]
(Only the edited continuation, wrapped in <polluted>...</polluted>.)
\end{lstlisting}

    \item \textbf{Teacher and Self-guidance generation prompt.} Including which context is shown (e.g., $(q,c_{0:\alpha})$ and both $(w^{\mathrm{clean}},w^{\mathrm{poll}})$) and the instruction to produce a brief appendable repair.

% In the document
\begin{lstlisting}[style=promptcode]
    [SYSTEM]
    You are a helpful reasoning assistant. Solve the user's problem.
    
    [USER]
    You are helping generate a short "repair continuation" for robustness training.
    
    You will be given:
    (1) a problem q
    (2) the student/agent's reasoning prefix c_0:alpha
    (3) two alternative next windows:
        - CLEAN window w_clean (the correct continuation)
        - POLLUTED window w_poll (a locally coherent but potentially wrong/misleading continuation)
    
    Your job:
    Write a SHORT continuation that would come *immediately after the POLLUTED window* and "repairs" the trajectory--i.e., it identifies/corrects the key mistake implied by w_poll and steers back toward the correct solution path.
    
    Critical constraints (important):
    - Output ONLY the repair continuation text (no headings, no bullet points, no meta commentary).
    - Do NOT mention the existence of w_clean or w_poll, and do NOT compare them explicitly.
    - The continuation must be natural as if the model only saw: (q, c_0:alpha, w_poll).
    - Keep it short (aim: ~1-4 sentences). Do not re-solve from scratch.
    - Prefer minimal intervention: point out the inconsistency, fix the specific step, then continue one step forward.
    - Maintain the same tone/style/level of detail as the prefix.
    - Do not output the final numeric answer unless the polluted context already forces you to (otherwise stop before the final answer).
    
    Now generate the repair continuation.
    
    [PROBLEM q]
    {q}
    
    [REASONING PREFIX c_0:alpha]
    {prefix}
    
    [CLEAN WINDOW w_clean -- for your reference only]
    {w_clean}
    
    [POLLUTED WINDOW w_poll -- this is what the agent saw at deployment]
    {w_poll}
    
    [YOUR OUTPUT: repair continuation that follows immediately after w_poll]
\end{lstlisting}

    % \item \textbf{Role-conditioning mechanism.} Specify how \textsc{pollute} vs \textsc{solve} is instantiated (e.g., special tokens, system strings, prefix tags), and confirm it is the same backbone $\pi_\theta$.
    % \item \textbf{Parsing rules.} State how generations are parsed into (i) corrupted windows, (ii) repair snippets, and (iii) final answers.
\end{itemize}

% -------------------------
\section{Diagnosability Evaluation Details}
\label{app:meta_reasoning}
\begin{itemize}[leftmargin=*]
    \item Evaluation prompt:

\begin{lstlisting}[style=promptcode]
[SYSTEM]
Act as a grade school math teacher and score the following problem solution.

[USER]
Question:
{data['question']}
Student Solution:
{sol_steps}
Your task involves three parts:
1. **Step-by-step Evaluation:** Go through the student solution carefully and identify
key errors and potential misunderstandings that led to the incorrect solution.
2. **Final Judgement:** Provide an overall judgement on the correctness of the
student's solution.
3. **First Error Step:** If the solution is incorrect, generate the step number where
the first error occurs, otherwise generate N/A here
4. **Error Analysis:** If the solution is incorrect, analyse the cause and reasons for
the first error step, otherwise generate N/A here
Here's the format I want:
Step-by-step Evaluation: [Provide a step by step examination of the student solution and
identify key errors and misunderstandings here.]
Final Judgement: [Insert only **correct** or **wrong** here]
First Error Step: [Insert either N/A or the step number where the first error occurs]
Error Analysis: [Insert either N/A or the analysis of error in the first error among
solution steps]
Please follow this format without any additional introductory or concluding statements.
\end{lstlisting}

\end{itemize}

% -------------------------
\section{Self-Revision Evaluation Details}
\label{app:rupbench}
\begin{itemize}[leftmargin=*]
    \item valuation prompt:

\begin{lstlisting}[style=promptcode]
[SYSTEM]
    A conversation between User and Assistant. The user asks a question, and the Assistant solves it.
    The assistant first thinks about the reasoning process in the mind and then provides the user with the answer.
    The reasoning process and answer are enclosed within <think> </think> and <answer> </answer> tags, respectively, i.e.,
    <think> reasoning process here </think><answer> answer here </answer>

[USER]
Question:
{QUESTION}

Your previous solution (incorrect):
{INCORRECT_SOLUTION}

Task:
1) Identify the key mistake(s) in the previous solution.
2) Produce a corrected solution.
3) Provide the final answer in the form boxed{answer} at the end of your response.
\end{lstlisting}

\end{itemize}

% -------------------------
\section{Hyperparameters and Implementation Details}
\label{app:hyperparams}
\paragraph{Optimization and precision.}
Unless otherwise stated, we optimize all models with AdamW (PyTorch implementation) using bfloat16 training and gradient checkpointing. We use a cosine learning-rate schedule with warmup ratio $0.1$, peak learning rate $3\times 10^{-6}$, weight decay $0$, and gradient-norm clipping at $1.0$. All runs use a set of fixed random seed (42, 52, 62).

\paragraph{Parameter-efficient fine-tuning.}
For the 8B backbone (Qwen3-8B), we use LoRA-based fine-tuning to fit training within limited hardware. LoRA adapters are applied to the attention and MLP projection layers with rank $r=64$ (shared across modules). All base model weights are frozen, and only LoRA parameters are updated during training. This configuration allows stable GRPO self-play training on 4$\times$A100 GPUs without model or tensor parallelism, while preserving sufficient adaptation capacity for robust reasoning and repair behaviors.

\paragraph{GRPO rollouts.}
We use Group Relative Policy Optimization (GRPO) with clipping parameter $\epsilon=0.2$ (Eq.~\ref{grpoObjective}). For each update, we sample $G=64$ trajectories per conditioning context at temperature $0.7$ (top-$k=50$, top-$p=1.0$). We cap the maximum prompt length to 1024 tokens and the maximum completion length to 1024 tokens. Unless otherwise specified, we use no explicit KL regularization to a reference policy (KL coefficient $\beta=0$). We use $\lambda = 0.07$ for the guidance term across all models except Qwen3-8B, where we use $\lambda = 0.04$. During training, we linearly anneal $\lambda$ to $0$ over the final steps.

\paragraph{Self-play scheduling.}
\textsc{GASP} alternates optimization between the agent and polluter roles instantiated by the same backbone model. We use a blocked update schedule: the agent performs 5 optimizer steps while holding the polluter fixed, then the polluter performs 5 optimizer steps while holding the agent fixed, repeating throughout training. When optimizing the polluter, the agent's rollout is treated as a black-box environment outcome (no backpropagation through the agent trajectory).

\paragraph{Implementation.}
We run training with vLLM-based generation for rollouts and standard distributed training utilities and modified based on TRL \citep{vonwerra2020trl}.

% -------------------------
\section{Representation Shift Measurements}
\label{app:repr_shift}

\paragraph{PCA-based representation shift.}
To quantify representational drift induced by different training procedures, we measure layer-wise shifts in hidden-state representations using principal component analysis (PCA), following prior work \citep{huan2025does}. 
Let $X$ denote a fixed probe set of inputs (e.g., TruthfulQA or CommonsenseQA). For each model state $\ast \in \{\text{base}, \text{post-training}\}$ and each transformer layer $i \in \{1,\dots,L\}$, we collect the hidden states
\[
H_i^{(\ast)} \in \mathbb{R}^{N \times d},
\]
where $N$ aggregates token representations over all inputs in $X$ and $d$ is the hidden dimension. We apply PCA with $n=2$ components to $H_i^{(\ast)}$ and compute the mean projection onto each principal axis:
\[
m^{(\ast)}_{i,k} = \mathrm{mean}\!\left( \mathrm{proj}_{\mathrm{PC}k}\!\left(H_i^{(\ast)}\right) \right), \quad k \in \{1,2\}.
\]

We define the PC1 shift relative to the base model as
\[
\Delta m^{(\ast)}_{i,1} = m^{(\ast)}_{i,1} - m^{(\text{base})}_{i,1},
\]
and form a 2D representation for each layer
\[
z^{(\ast)}_i = \bigl(\Delta m^{(\ast)}_{i,1},\; m^{(\ast)}_{i,2}\bigr).
\]
Each point in the PCA shift plots (e.g., Fig.~5) corresponds to one transformer layer.

To summarize global representational movement, we compute the centroid across layers
\[
z^{(\ast)} = \frac{1}{L} \sum_{i=1}^{L} z^{(\ast)}_i,
\]
and report the scalar PCA shift magnitude
\[
d^{(\ast)} = \left\lVert z^{(\ast)} - z^{(\text{base})} \right\rVert_2.
\]
Larger $d^{(\ast)}$ indicates greater global drift of internal representations on the probe distribution.

\end{document}